\documentclass[11pt]{article}
\usepackage{preamble}
\newcommand{\nc}{\newcommand}
\newcommand{\rnc}{\renewcommand}

\newcommand{\bm}[1]{\textbf{\textit{#1}}}
\nc{\on}[1]{\operatorname{#1}}
\nc{\ol}[1]{\overline{#1}}
\nc{\wt}[1]{\widetilde{#1}}
\nc{\vp}{\varphi}
\nc{\eps}{\varepsilon}

\newcommand\restr[2]{{
  \left.\kern-\nulldelimiterspace
  #1 
  \vphantom{\big|}
  \right|_{#2}
}}

\nc{\R}{\mathbb R}
\nc{\C}{\mathbb C}
\nc{\Q}{\mathbb Q}
\nc{\Z}{\mathbb Z}
\nc{\N}{\mathbb N}
\nc{\F}{\mathbb F}
\nc{\D}{\mathbb D}
\nc{\ind}{\mathbf{1}}

\let\E\relax
\DeclareMathOperator{\E}{\mathbb{E}}
\let\Pr\relax
\DeclareMathOperator*{\Pr}{\mathbb{P}}

\nc{\Proj}{\operatorname{Proj}}

\nc{\CH}{\mathcal H}
\nc{\CL}{\mathcal L}
\nc{\CD}{\mathcal D}
\nc{\CZ}{\mathcal Z}
\nc{\CP}{\mathcal P}
\nc{\CF}{\mathcal F}
\nc{\CG}{\mathcal G}
\nc{\CN}{\mathcal N}
\nc{\CX}{\mathcal{X}}
\nc{\CY}{\mathcal{Y}}
\nc{\CU}{\mathcal{U}}
\nc{\CA}{\mathcal{A}}
\nc{\CE}{\mathcal{E}}
\nc{\CI}{\mathcal{I}}
\nc{\CC}{\mathcal{C}}
\nc{\CR}{\mathcal{R}}
\nc{\CS}{\mathcal{S}}
\nc{\CB}{\mathcal{B}}

\nc{\defn}[1]{\textit{#1}}
\nc{\eg}{\emph{e.g.,~}}
\nc{\ie}{\emph{i.e.}}

\nc{\Lzo}{\mathcal{L}_{0/1}}
\nc{\loss}{\mathcal{L}}
\nc{\lambdamin}{\lambda_{\min}}
\nc{\lambdamax}{\lambda_{\max}}
\rnc{\t}[1]{\text{#1}}
\nc{\sgn}{\on{sgn}}

\usepackage{color-edits}
\addauthor[Dutch]{dh}{ForestGreen}
\addauthor[to-do]{t}{red}
\addauthor[Jerry]{jl}{purple}
\addauthor[Scott]{sg}{cyan}

\title{Weak-to-Strong Generalization is Nearly Inevitable (in Linear Models)}
\author{
    Scott Geng\thanks{University of Washington, Seattle. \texttt{sgeng@cs.washington.edu}.} \and
    Dutch Hansen\thanks{University of Washington, Seattle. \texttt{jmhans13@cs.washington.edu}.} \and
    Jerry Li\thanks{University of Washington, Seattle. \texttt{jerryzli@cs.washington.edu}.}
}
\date{\today}

\begin{document}

\maketitle
\begin{abstract}
    Weak-to-strong generalization is a phenomenon in post-training whereby a strong student model, when finetuned solely with feedback from a weaker teacher, can not only surpass the teacher, but can improve upon its own capabilities. Recent work of~\cite{burns2023weak} demonstrated that this can occur in the setting of frontier language models, and subsequently there has been a flurry of both empirical work trying to exploit this phenomenon, as well as theoretical work attempting to understand it.
In this work, we demonstrate that weak-to-strong generalization occurs in standard linear logistic regression, under mild distributional assumptions on the data.
In fact, we show that this happens for \emph{most} student-teacher pairs, suggesting that weak-to-strong generalization is in fact \emph{almost inevitable}, even in this basic setting.
Notably, our setting does not require the student to be more expressive or have more model capacity in any way compared to the teacher, which runs contrary to the prevailing theoretical belief that a mismatch in model capacity is a central mechanism to weak-to-strong generalization.
\end{abstract}

\thispagestyle{empty}
\setcounter{page}{0}

\newpage
\section{Introduction}\label{sec:introduction}

We consider the fascinating but mysterious phenomenon of \emph{weak-to-strong} generalization in machine learning.
Typically, to improve the performance of a student model, we expect to require guidance either from a stronger teacher model, or from the ground truth.
Yet empirically, it has been observed that a student model can \emph{surpass} its teacher and its own prior capabilities.\footnote{We note that in the theoretical literature, the term ``weak-to-strong'' generalization sometimes refers to a slightly weaker notion than the one considered in this paper, where the student converges to a model which is stronger than the teacher, but may be weaker than the original student, see e.g.~\cite{charikar2024quantifying,medvedev2025weak}.
However, weak-to-strong generalization for us will mean that the student's performance \emph{improves upon itself}. Note that this is what was observed empirically in post-training frontier models by~\cite{burns2023weak}.}
This occurs even though the teacher model is the only source of feedback for the student, and even when the student is blindly following the advice of the teacher.
This flies in the face of conventional wisdom in machine learning, which would suggest that the weaker teacher model should drag down the performance of the student.

While such phenomena have been anecdotally documented since the dawn of machine learning~\cite{samuel1959some}, it has recently gained traction in the context of post-training of top-of-the-line frontier language models.
In a striking experiment,~\cite{burns2023weak} demonstrated that if you post-train a strong pre-trained model such as GPT-4 with a weaker teacher such as GPT-2 using naive supervised finetuning, then for many tasks, the performance of the final post-trained model exceeds both the performance of the base pre-trained model as well as that of the teacher.
In particular, this implies is that there is a point during the training process at which the student model has better performance than the teacher (i.e., it is a stronger model than the teacher), and yet, it improves solely by using feedback from the weaker teacher model.
They go on to argue that this effect can enable us to bootstrap the performance of state-of-the-art models using smaller, weaker models, and further speculate that in a future where AI systems achieve super-human capabilities, this suggests that human feedback can still be used to guide these systems, despite the fact that the human teachers are weaker than the learned AI system.

\begin{figure}[ht]
\centering
\begin{minipage}{0.6\textwidth}
    \includegraphics[width=\linewidth]{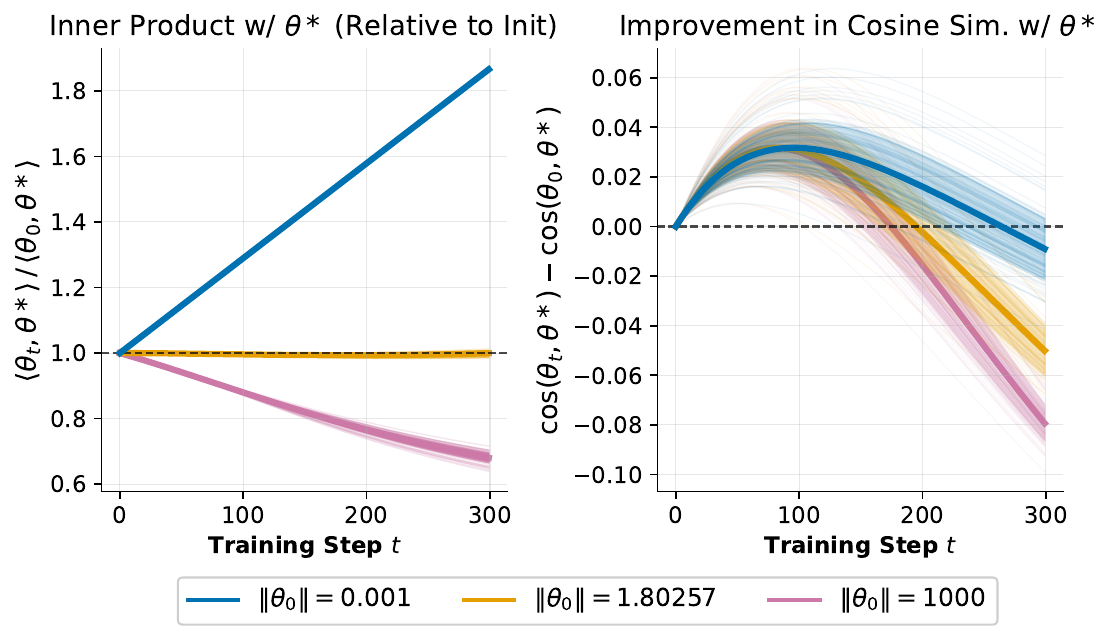}
\end{minipage}%
\hfill
\begin{minipage}{0.385\textwidth}
    \refstepcounter{figure}%
    {\small Figure \thefigure: We train an 80\%-accurate student linear model $\theta_0$ on pseudo-labels from a randomly drawn strictly weaker teacher, from the same family (70\% acc.). Over 100 teacher draws (thin lines), we always observe the student improve despite the teacher's bad advice. Learning dynamics depend on the initial norm $\|\theta_0\|$: notably, weak-to-strong gains occur even when the teacher steers the student away from the true model $\theta^*$ (purple traces)! Thick lines and shading show mean $\pm$~std.}
    \label{fig:teaser}
\end{minipage}
\end{figure}

While weak-to-strong generalization has been well-documented empirically, our theoretical understanding of why and how it occurs is still developing.
To oversimplify the state of affairs (see Section~\ref{sec:related-work} for a more careful discussion), the prevailing hypothesis appears to be that weak-to-strong generalization primarily occurs when the student model has more model capacity, and thus expressive power, than the teacher model (see e.g.~\cite{charikar2024quantifying,mulgund2025relating,lang2024theoretical,medvedev2025weak,wu2024provable,dong2025discrepancies}), or at the very least, has different expressive power than the teacher, such as in classic settings such as co-training~\cite{blum1998combining,balcan2004co}.
But is a difference in model capacity truly the only way that weak-to-strong generalization can occur?
Or in other words:
\begin{center}
    {\it Can weak-to-strong generalization occur even in settings with absolutely no mismatch in expressive power between the teacher and student?}
\end{center}
In this work, we show that, perhaps surprisingly, the answer to this question is yes.
We demonstrate that, under mild assumptions, weak-to-strong generalization can occur even in the extremely classical and well-studied setting of logistic regression for Gaussian data.
In fact, we show that if the student and teacher are sampled randomly from a natural ensemble of models, then weak-to-strong generalization is \emph{almost inevitable} (\cref{fig:teaser})!
We also show that our results extend to a large class of other data distributions satisfying a natural condition called \emph{approximate ellipticity}.
Concretely, we consider the following setup, which is the natural instantiation of supervised finetuning (SFT) for logistic regression:
\begin{definition}[SFT for logistic regression]
\label{def:sft}
    Let $\mu$ be a data distribution over $\R^d$, and let $\theta^* \in \R^d$ be the ground truth classifier.
    For any linear classifier $\phi \in \R^d$, the $0/1$ loss of the classifier is defined as 
    \[
        \Lzo (\phi) := \Pr_{\bm{x} \sim \mu} [\on{sgn} (\phi^\top x) \neq \on{sgn} ((\theta^*)^\top x)] \; .
    \]

    Let $\theta_0 \in \R^d$ be a student model, and let $\psi \in \R^d$ be a teacher model.
    Let $\beta > 0$ be a temperature parameter, $T$ be the total number of steps of post-training, and let $B$ be the batch-size of the update.
    Then, supervised finetuning of $\theta_0$ with $\psi$ is the stochastic process defined by $\hat\theta_0 = \theta_0$, and $\hat\theta_{t+1} = \hat\theta_t - \eta g_t$\, for $t \in \{0, \ldots, T-1\}$, where 
    \begin{align*}
        g_t := \frac{1}{B} \sum_1^B \nabla \loss_\beta (\hat\theta_t, \bm{x}_i^{(t)}, \bm{y}_i^{(t)}),
    \end{align*}
    where $\{\bm{x}_i^{(t)}\}_{t,i} \sim_{\on{iid}} \mu$, and $\bm{y}_i^{(t)} = \ind(\psi^\top \bm{x}_i^{(t)} \geq 0)$ are labels from the teacher, and $\loss_\beta$ is the logistic loss with temperature $\beta$ (see \Cref{def:logistic-loss}).
\end{definition}
\noindent We note that since throughout this process, we only see the signs of $\psi^\top \bm{x}$ and ${\theta^*}^\top \bm{x}$, so we take $\psi$ and $\theta^*$ to be unit vectors without loss of generality throughout this paper.

With this setup, a high level statement of the result that we show can be summarized as follows.
Here and throughout the paper, we will always let $\| \cdot \|$ denote the $\ell_2$ norm, and for two vectors $u, v$, we let $\cos (u, v) = \langle u, v \rangle / (\| u \| \| v \|)$ denote their cosine similarity.
\begin{theorem}[informal, see Theorem~\ref{thm:stochastic-analysis}]
    Let $\mu$ be a sufficiently nice data distribution over $\R^d$.
    Suppose the following conditions hold:
    \vspace{-3pt}
    \begin{enumerate}
        \item \textbf{The teacher must be better than random guessing:} we have that $\Lzo (\psi) < c_1$, for some universal constant $c_1 \leq 1/2$.
            \vspace{-5pt}
        \item \textbf{The student and teacher are poorly correlated orthogonal to $\theta^*$:} we have that 
        \[
        |\cos (\psi, \on{Proj}_{(\theta^*)^\perp} \theta_0)| \leq c_2 \; ,
        \]
        for some universal constant $c_2$, where for any vectors $u, v \in \R^d$, we let $\on{Proj}_{u^\perp} v$ be the projection of $v$ onto the subspace orthogonal to $u$.
    \end{enumerate}
    Then, there exists $\delta = \on{poly} (c_1, c_2)$ so that we have that with high probability, SFT on $\theta_0$ with $\psi$ as in Definition~\ref{def:sft} with constant step-size will yield
    \[
    \Lzo (\theta_0) - \Lzo (\theta_T) \geq \Omega\left(\frac{\rho}{\|\theta_0\| + 1}\right)^2 \; ,
    \]
    as long as $T = \on{poly}(\delta^{-1}, \| \theta_0 \|)$ and $B = d \cdot \on{poly}(\delta^{-1}, \|\theta_0\|)$.
\end{theorem}

We pause here to interpret the theorem and the important conditions therein.

\paragraph{Regularity conditions on $\mu$.} 
We go into more detail about the conditions we require on the distribution $\mu$ in Section~\ref{sec:overview}, but the main assumption on $\mu$ we will require is a notion of approximate sphericity or ellipticity which generalizes previously studied notions of ellipticity for high-dimensional distributions, as well as a mild smoothness condition on the pdf of $\mu$.
These assumptions are satisfied by Gaussians, as well as any spherically symmetric or approximately spherically symmetric distribution with reasonable tails.
For simplicity, on first reading, the reader is encouraged to think of the special case where $\mu$ is the standard normal Gaussian.

\paragraph{A not terrible teacher} The next condition we need is that the teacher is not actively harmful, but otherwise, our assumptions on the teacher's power is very mild: we just need that it achieves $0/1$ loss better than random guessing  by only some small constant.
Note that in the setting of Gaussian data, this just corresponds to the teacher having positive cosine similarity with the ground truth $\theta^*$.

\paragraph{Non-aligned students and teachers} The final condition we require is that the student and the teacher are poorly correlated in the directions orthogonal to the true direction $\theta^*$.
If we think of the directions orthogonal to $\theta^*$ as noise directions, then a natural interpretation of this condition is that the sources of error for the teacher and the student models are not too similar.
This is a natural condition, and one that is easily seen to be necessary for this setting.
For instance, if the student and teacher are random in the directions orthogonal to $\theta^*$, then this quantity goes to $0$ as $d \to \infty$.
Moreover, we also empirically demonstrate that if we train two models from random initialization on overlapping data, then the resulting models end up being quite uncorrelated in practice (see Section~\ref{sec:exp}).

\paragraph{The mechanism for weak-to-strong generalization for logistic regression.}
The theorem then states that under these conditions, and as long as the student is initially neither too good nor too terrible, then SFT with a weaker teacher can provably improve the performance of the student.
We note that the improvement that can be achieved is limited: one cannot typically hope to achieve vanishingly small 0/1 error with this setup, and in our simulations, we indeed see that SFT with a weak teacher improves performance by a small but non-trivial constant (see Figure~\ref{fig:teaser} or Section~\ref{sec:exp}).

Note that unlike most previous setups, there is no assumption about the relative expressive power between the student and the teacher!
Indeed, in our setting, both the student and teacher models are fairly unstructured, and so is the data distribution.
So why does weak-to-strong generalization occur?

To explain this phenomenon, consider the simplified setting where the data distribution is the standard normal Gaussian, i.e. $\mu = \CN(0, I)$.
In this case, improving the zero-one loss of a model corresponds to simply increasing its cosine similarity with $\theta^*$.
Now, suppose we move a student model $\theta$ infinitesimally along a line $\ell (t) = \theta + t v$ for some vector $v$.
A straightforward calculation (see e.g.~\cite{geng2025delta}, Equation (23)) shows that if we let $f(t) = \cos (\ell(t), \theta^*)$ denote change in cosine similarity with $\theta^*$ as we vary $\theta$ along this line, then
\[
\| \theta \| \cdot f'(0) = \underbrace{\vphantom{\frac{\cos (\theta, \theta^*)}{\| \theta \|}}\left( 1 - \cos^2 (\theta, \theta^*) \right) \cdot \left( v^\top \theta^* \right)}_{A(v)}  - \underbrace{ \frac{\cos (\theta, \theta^*)}{\| \theta \|} \cdot \left\langle \on{Proj}_{(\theta^*)^\perp} \theta, \on{Proj}_{(\theta^*)^\perp} v \right\rangle}_{B(v)} \; .
\]
Here the terms $A$ and $B$ have natural interpretations: $A$ represents the change in the cosine similarity due to the change in $\theta$ in the direction of $\theta^*$, and $B$ represents the change in the cosine similarity due to the movement of $\theta$ in the directions orthogonal to $\theta^*$.
In particular, for us to make progress as we move infinitesimally along this line, we need that $f'(0) > 0$.

Now, consider what happens when $v$ is the update direction induced by SFT.
For simplicity, we take the temperature parameter $\beta = 1$, and we take gradient steps with respect to the population gradient, as opposed to sample gradients.
In this case, the SFT update direction $v$ (i.e., negative gradient) for a student model $\theta$ with a teacher $\psi$ is given by
\begin{align*}
v &= \E_{\bm{x}} \left[ \left( \ind (\psi^\top x \geq 0) - \sigma (\theta^\top \bm{x}) \right) \bm{x} \right] = \psi - m (\| \theta \|) \cdot \frac{\theta}{\| \theta \|} \; ,
\end{align*}
where $m(x)$ is some bounded function so that $m (x) \geq 0$ for $x \geq 0$.
The exact form of $m$ can easily be derived but is mostly irrelevant for this discussion (of course, we will need to handle it rigorously later on).

For this choice of $v$, we have that
\[
A(v) = \left( 1 - \cos^2 (\theta, \theta^* )\right) \cdot \left( \cos (\psi, \theta^*) - m(\| \theta \| ) \cos (\theta, \theta^*) \right)  \; .
\]
If the norm of $\theta$ is sufficiently small, it turns out that this is positive, and so we are moving in the direction of $\theta^*$, so it is perhaps not surprising that we can make progress in this case (e.g., the blue lines in Figure~\ref{fig:teaser}).
But, if the norm of $\theta$ is at least some small constant, and if the teacher is weaker than the student, i.e. the teacher's cosine similarity with $\theta^*$ is smaller than the student's cosine similarity, it can easily be the case that $A(v) < 0$, that is, we are actively decreasing our correlation with $\theta^*$.
Yet we claim that even here, we still make progress with SFT!

This is because the second term $B(v)$, i.e. the change in cosine similarity due to moving $\theta$ in the directions orthogonal to $\theta^*$, turns out to help us in the setting.
Note that this can only happen if the SFT update is decreasing the norm of $\theta$ in the directions orthogonal to $\theta^*$.
In fact, for the Gaussian setting, it more or less forces the overall gradient to always be positive, under our assumptions!
If $\psi$ and $\theta$ are more or less orthogonal to each other in the subspace orthogonal to $\theta^*$, as we assume, then
\begin{align*}
B(v) &\approx - m(\| \theta \|) \cdot \cos (\theta, \theta^*) \cdot \frac{\| \on{Proj}_{(\theta^*)^\perp} \theta \|^2}{\| \theta \|^2} \\
&= - m(\| \theta \|) \cdot \cos (\theta, \theta^*) \cdot \left( 1 - \cos^2 (\theta, \theta^* )\right) \; ,
\end{align*}
and so the overall expression for $f'(0)$ is
\begin{align*}
\| \theta \| \cdot f'(0) &\approx A(v) + m(\| \theta \|) \cdot \cos (\theta, \theta^*) \cdot \left( 1 - \cos^2 (\theta, \theta^* )\right) \\
&= \left( 1 - \cos^2 (\theta, \theta^* )\right)  \cos (\psi, \theta^*) > 0 \; .
\end{align*}
Thus, because the norm of $\theta$ is shrinking sufficiently quickly in all directions orthogonal to $\theta^*$ (i.e. all the noise directions), our cosine similarity with $\theta^*$ actually improves, even though the update direction has negative correlation with $\theta^*$! See the purple lines in Figure~\ref{fig:teaser} for empirical illustration.
We can summarize this mechanism with the following tongue-in-cheek aphorism:
\begin{center}
    {\it Given feedback from a bad advisor, a good student will unlearn their own mistakes faster than they will be misled by the advisor's stupidity.}
\end{center}
This will, of course, be of great comfort to many professors throughout academia.

Eventually, this effect must disappear, as eventually, the SFT procedure will converge to the global minimum, which is the labels given by the weak teacher, but not before we have made some non-trivial progress at some intermediate iteration.
In the full paper, we formally prove this effect occurs in a much more general setting beyond Gaussian inputs, but for morally similar reasons.
To show the full theorem, we then demonstrate that this effect still occurs even when we discretize the trajectory, and also even if we use stochastic gradients.
We go over some of these calculations in more detail in Section~\ref{sec:overview}, but we defer the full technical details to the supplementary material.

\paragraph{The role of the student's confidence.} In our theorem statement, our final improvement in the $0/1$-loss, as well as the batch-size and number of steps of SGD, depends inversely on $\| \theta_0 \|$.
Intuitively, in logistic regression, the norm of a model corresponds to its relative confidence: a larger norm model has higher confidence, whereas a lower norm model tends to yield lower confidence scores.
So our results would suggest that the relative gain in SFT is smaller for more confident models.
However, we conjecture that the dependence on $\| \theta_0 \|$ is actually unnecessary, and is a mostly relic of our techniques.
In our simulations, we find that in fact if you run SFT with sufficient hyperparameters, then large norm students make essentially the same amount of progress as small norm students.
Showing that this happens however requires more nuanced understanding of the long time-horizon evolution of SFT, which is beyond the scope of this paper.

\subsection{Related Work}
\label{sec:related-work}

There is by now a vast empirical literature on techniques for post-training frontier models, including many popular techniques such as supervised finetuning, alongside more complex techniques such as reinforcement learning~\cite{ouyang2022training,bai2022training} and various preference optimization and contrastive learning techniques~\cite{schulman2017proximal,dong2023raft,rafailov2023direct}, just to name a few.
We defer the interested reader to \emph{e.g.}~\cite{kumar2025llm} for a more extensive review of this large and ever-shifting literature.
The most relevant paper from this literature is arguably the aforementioned paper of~\cite{burns2023weak} which highlighted the phenomenon of weak-to-strong generalization in a real-world frontier model.
Subsequently, several papers have explored how to exploit this phenomenon for post-training strong models, see e.g.~\cite{li2024superfiltering,shin2024weak,somerstep2024transfer,bansal2024smaller,yang2024weak,sun2024easy,ji2024aligner,yang2024super,guo2024improving,lang2025debate,geng2025delta,nie2025weak,ye2025weak}.

On the more theoretical side, there has been a flurry of recent work attempting to understand the phenomenon of weak-to-strong generalization.
A large literature has begun to focus on the the hypothesis that differences in model capacity play a crucial role, under a variety of different models, see e.g.~\cite{mobahi2020self,lang2024theoretical,shin2024weak,charikar2024quantifying,wu2024provable,medvedev2025weak,mulgund2025relating,oh2025linear,yao2025capabilities,wilop2025predicting,dong2025discrepancies,yao2025weak}.
These results often require that the student and the teacher have some measure of disagreement, see e.g.~\cite{charikar2024quantifying,mulgund2025relating,lang2024theoretical,wu2024provable,dong2025discrepancies}.
But in these prior works, these properties are typically exploited in conjunction with a difference in model capacity.
Our results suggest that teacher/student disagreement may already suffice by itself to explain weak-to-strong generalization. Other explanations have been proposed as well. The work of ~\cite{zhang2024transcendence} posits that temperature may play a key role.
The work of~\cite{lang2024theoretical} points toward a general pseudolabel-correction phenomenon;
they show that a sufficiently robust student model can correct the mistakes of a weaker teacher, under some expansion condition on the teacher errors.
We note that several of these results~\cite{lang2024theoretical,yao2025capabilities,yao2025weak} do not consider the training dynamics, proving point-wise error comparisons or assuming that we can find an optimal solution to the SFT process.
In contrast, our results crucially require that we do not find a globally optimal solution to the finetuning process (and instead perform some kind of early stopping).

Finally, we remark that this setting has some resemblance to the classical literature on co-training, see e.g.~\cite{blum1998combining,balcan2004co}, where several weak models that predict using different features can be boosted to form a stronger model.
This can be viewed as a form of weak-to-strong generalization, where we treat one of the models as a student model, and the others as teacher models.
As in our setting, a key mechanism which enables this boosting is a notion of uncorrelated errors.
However, in these settings, it is typically assumed that the models have different views of the data, i.e. they act on different features of the data.
In other words, the models have different expressive powers.
In contrast, we seek to understand the setting in which the teacher and student models use exactly the same features of the data.

\section{Technical Overview}
\label{sec:overview}

\subsection{Preliminaries}

For a positive-definite matrix $M$, we let $\langle x, y\rangle_M$ denote the inner product given by $x^\top M y$ and $\|x\|_M := \sqrt{\langle x, x\rangle_M}$ denote the induced norm. For a matrix $A$, we write the induced operator norm on this space $\|A\|_M$. The associated cosine similarity is denoted $\cos_M(x, y) := \langle x, y\rangle_M/\|x\|_M\|y\|_M$. For a subspace $E$ of $\R^d$, we denote by $\on{Proj}_E^M$ the orthogonal projection onto $E$ with respect to $\langle \cdot, \cdot\rangle_M$. For brevity, we sometimes drop the subscript/superscript when using Euclidean inner product ($M = I$). Throughout, when we write $\lambdamax$, $\lambdamin$, we refer to the top and bottom eigenvalues of the covariance $\Sigma$, which will be clear from context.

\begin{definition}[Logistic loss]\label{def:logistic-loss}
    We denote by $\sigma_\beta$ the logistic function of inverse temperature $\beta > 0$: \begin{align*}
        \sigma_\beta(x) := \frac{1}{1+e^{-\beta x}}.
    \end{align*}
    For $\theta, x \in \R^d$ and $y \in \{0,1\}$, we define the associated logistic loss: \[\loss_\beta(\theta; x, y):= -\log p_\theta(y|x), \quad \text{where} \quad  p_\theta(1|x) = \sigma_\beta(\theta^\top x).
    \]
\end{definition}

\noindent For $\bm x \sim \mu$ and $\bm y = \ind[\psi^\top \bm{x} \geq 0]$, \begin{align*}
    \nabla_\theta \loss_\beta(\theta;x, y) = \beta \sigma_\beta(\theta^\top x)x - \beta yx, \quad \quad \nabla \loss_\beta(\theta) = \beta \cdot \E\left[ \sigma_\beta(\theta^\top \bm{x})\bm{x} - \ind[\psi^\top {\bm x} \geq 0] {\bm x} \right].
\end{align*}
Note that $\sigma_\beta' = \beta \cdot \sigma_\beta(1-\sigma_\beta) \leq \beta/4$, so $\sigma_\beta$ is $\beta/4$-Lipschitz. As a consequence, the population loss is $\beta^2 \lambdamax/4$-smooth. Taking the Hessian, \begin{align*}
    \nabla^2 \loss_\beta(\theta) = \beta \cdot \E\Big[ \sigma_\beta'(\theta^\top \bm{x}) \bm{x}\bm{x}^\top \Big] \preceq \frac{\beta^2}{4} \Sigma.
\end{align*}

\subsection{Approximate Ellipticity}\label{sec:approximate-ellipticity}

It is known that the low-dimensional projections of high-dimensional distributions, under mild assumptions, are approximately Gaussian \cite{diaconis1984asymptotics,duembgen2011low}. Starting with the seminal work of \cite{hall1993almost}, a closely related line of inquiry revealed that such low-dimensional projections also possess strong characteristics of elliptical symmetry \cite{leeb2013conditional,steinberger2018conditional}. Inspired by these observations, our analysis makes use of a an approximate ellipticity notion with tight connections to this literature. We discuss these connections in \Cref{sec:approximate-ellipticity-assumption}.

\begin{definition}\label{def:ellipticity}
    A random variable $\bm{x}$ over $\R^d$ with mean zero and positive-definite covariance $\Sigma$ is said to be $\eps$-elliptical along a (nontrivial) subspace $E$ if for all $v \in E \setminus \{0\}$ and $u \in \R^d$ such that $u^\top \Sigma v = 0$, it holds that \begin{align*}
        \Big| \E\left[ u^\top \bm{x} \, | \, v^\top \bm{x} \right] \Big| \leq \eps \cdot \|u\|_{\Sigma} \quad a.s.
    \end{align*}
    When $E = \R^d$, we just say that $\bm{x}$ is $\eps$-elliptical. 
\end{definition}

\noindent Note that for $v \in E \setminus \{0\}$ and any $u$, we can also decompose $u$ into its orthogonal projections with respect to the inner product induced by $\Sigma$. It is therefore equivalent to require that for all $u$, \begin{align*}
    \left| \E\left[ u^\top \bm{x} \, | \, v^\top \bm{x} \right] - \frac{u^\top \Sigma v}{v^\top \Sigma v}\cdot v^\top \bm{x} \right| \leq \eps \cdot \|u\|_{\Sigma} \quad a.s.
\end{align*}
One also observes that the property is preserved under full-rank transformations, albeit with a change of the relevant space $E$: if $\bm{x}$ is $\eps$-elliptical along $E$, then $T\bm{x}$ is $\eps$-elliptical along $T^{-\top} E$. As a special case, transforming by $\Sigma^{-1/2}$ gives $\eps$-ellipticity along $\Sigma^{1/2} E$ with covariance $I$. The isotropic setting will be especially useful; when $\Sigma = I$ in \Cref{def:ellipticity}, we will say that $\bm{x}$ is $\eps$-spherical along $E$.

Using the characterization of \cite{eaton1986characterization}, one notes that if $\on{cov}\bm{x}=I$, $E = \R^d$, and $\eps = 0$, then $\bm{x}$ is spherically symmetric in the sense that its law is invariant under unitary transformations. The following computation follows naturally.

\begin{lemma}
    \torestate{\label{lem:arccos-loss}
    Let $\bm{y}$ be a $0$-elliptical random vector on $\R^k$ with covariance $\Sigma$ and Lebesgue density. The for nonzero $\phi, \psi \in \R^d$: \begin{align*}
        \Pr[\, \on{sgn} \phi^\top \bm{y} \neq \on{sgn} \psi^\top \bm{y} \,] = \frac{1}{\pi} \arccos \circ \cos_\Sigma(\phi, \psi).
    \end{align*}
    }
\end{lemma}

While we view approximate sphericity on small subspaces as a relatively mild assumption (see \Cref{sec:approximate-ellipticity-assumption}), one can use this property to obtain a similar characterization of the zero-one loss. To do so, we go through the Fourier transform. For sufficiently nice $f : \R^d \to \R$: \begin{align*}
    \hat{f}(\xi) = \CF[f](\xi) := \int_{\R^d} f(x) \, e^{-i2\pi \langle x, \xi\rangle} dx.
\end{align*}
Here, we adopt the unitary convention.

\begin{proposition}
    \torestate{\label{prop:arccos-loss-approx}
    Let $\bm{x} \sim \mu$ be a random vector in $\R^d$, $d \geq 2$, with covariance $\Sigma$ and Lebesgue density satisfying \begin{align*}
        \Big| \widehat{\frac{d\mu}{dm}}(\xi)\Big| \leq C(1+\|\xi\|)^{-k},
    \end{align*} for constant $C \geq 1$ and integer $k \geq 3$. Suppose $\phi, \psi$ are linearly independent and $\bm{x}$ is $\eps$-elliptical along $\on{span}\{\phi, \psi\}$. Then \begin{align*}
        \Big| \Pr[\,\on{sgn} \phi^\top \bm{x} \neq \on{sgn} \psi^\top \bm{x}\,] -\frac{1}{\pi} \arccos \circ \cos_\Sigma(\phi, \psi) \Big| \lesssim C^{\frac{3}{k+1}} \max\{1, \lambda_{\max}^{3/2}\} \cdot \eps^{\frac{1}{2}-\frac{3}{2(k+1)}}.
    \end{align*}
    }
\end{proposition}

\noindent In light of the tight connections between smoothness and Fourier decay, we view the density assumption as somewhat natural. To obtain this estimate, we apply an approximate version of Eaton's ellipticity characterization on $\on{span}\{\phi, \psi\}$. After showing approximate rotational symmetry of the frequencies, we use a spherical averaging procedure to obtain TV closeness to some idealized distribution on a constant-dimensional subspace. It is then enough to invoke \Cref{lem:arccos-loss}. In proving this lemma, we actually do not rely on full $\eps$-ellipticity along $\on{span}\{\phi, \psi\}$; we only require that the projection of our distribution onto this subspace be $\eps$-elliptical. Crucially, \Cref{prop:arccos-loss-approx} furnishes a target under which we can study the dynamics of SGD under weak supervision; since $\arccos$ is strictly monotone decreasing on $(0,1)$ (in fact has a monotonic derivative on $[1/2, 1)$), we may restrict our attention to quantifying the change in cosine similarity.

\subsection{Sufficient Conditions for Generalization}\label{sec:sufficient-conditions}

To start, we describe the conditions that enable our analysis. For clarity of presentation, we will assume $\on{cov} \bm{x} = I$ for the remainder of \Cref{sec:overview}, and leave the generalized analysis to \Cref{sec:full-results}. We will further assume that $\|\theta^*\| = \|\psi\| = 1$ as they only interact with the learning process via their direction. We define the confidence measure \begin{align*}
    m_{\mu, \beta}(\theta_0) := \frac{1}{\|\theta_0\|} \E \left[\,\sigma_\beta(\theta_0^\top\bm{x}) \cdot \theta_0^\top\bm{x}\right].
\end{align*}
Note that by the mean zero and identity covariance assumptions, $m_{\mu, \beta} \in [0,1]$. Note also that this quantity is increasing in both $\beta$ and $\|\theta_0\|$. 

We observe that learning from a weak teacher is possible when a certain ``relative correlation'' quantity---which compares the cosine similarities of the student and the teacher after an adjustment by $m_{\mu, \beta}(\theta_0)$---is positive. We first consider two relevant terms: \begin{align*}
    \rho_\ell &:= \frac{1}{2} \E|\psi^\top\bm{x}| \cdot \cos(\theta^*, \psi) - m_{\mu, \beta}(\theta_0) \cdot \cos(\theta^*, \theta_0)-2\eps,\\
    \rho_\upsilon &:= m_{\mu, \beta}(\theta_0) \cdot \|\on{Proj}_{(\theta^*)^\perp}\tilde{\theta}_0 \|^2-\frac{1}{2} \E|\psi^\top\bm{x}|\cdot |\cos( \on{Proj}_{(\theta^*)^\perp} \theta_0,\psi)|-2\eps.
\end{align*}
One can view $\rho_\ell$ as the student's potential to learn from the teacher's knowledge of $\theta^*$, and $\rho_\upsilon$ as the student's potential to unlearn its own orthogonality with $\theta^*$. The coefficients $\E|\psi^\top\bm{x}|$ and $m_{\mu, \beta}(\theta_0)$ can be seen as adjusting the correlations and confidence measure to the scale of the marginals at hand. Central to our analysis will be the following quantity:
\begin{align}\label{eq:overview-signal-term}
    \rho := (1-\cos^2(\theta_0, \theta^*))\rho_\ell + \cos(\theta_0, \theta^*)\rho_\upsilon.
\end{align}
We observe that when $\bm{x}$ is $\eps$-elliptical along $\on{span}\{\theta_0, \psi\}$, then $\rho > 0$ is sufficient for the student to learn by following the pseudo-labels of $\psi$. In the present setting, where $\on{cov}\bm{x} = I$, we notice that there is actually cancellation in \Cref{eq:overview-signal-term}, and \begin{align}\label{eq:overview-signal-term-cancellation}
    \rho &\geq \frac{\E|\psi^\top\bm{x}|}{2} \Big[ (1-\cos^2(\theta_0, \theta^*)) \cos(\psi, \theta^*)- \cos(\theta_0, \theta^*) |\cos( \on{Proj}_{(\theta^*)^\perp} \theta_0,\psi)| \Big] - 4\eps.
\end{align}
As described in the introduction, the condition that $\rho > 0$ is then easily satisfied by high-dimensional distributions of interest.

\subsection{Optimization Guarantees}

We now state our main result, restricted to the setting where $\on{cov} \bm{x} = I$ and $\beta = 1$. 
\begin{theorem}\label{thm:isotropic-result}
    Let $\mu$ be $\eps$-elliptical along $\on{span}\{\psi, \theta_0\}$. Suppose $\cos(\theta_0, \theta^*) \geq 0$ and $\rho > 0$ from \Cref{eq:overview-signal-term}. Then for some $T := \tau/\eta$ and any batch size $B$, where \begin{align*}
        \tau = O \left(\frac{\rho}{\|\theta_0\|}\right) \quad \quad \eta = \Omega\left( \frac{\rho}{\|\theta_0\| + 1 + \eps}\right)^3,
    \end{align*}
    it holds that \begin{align*}
        \cos(\hat{\theta}_T, \theta^*) - \cos(\theta_0, \theta^*) \geq \Delta := \Omega \left( \frac{\rho}{\|\theta_0\| + 1 + \eps}\right)^2,
    \end{align*}
    with failure probability at most \begin{align*}
        \frac{d}{B\Delta}\cdot O\left( \frac{1}{\|\theta_0\|\rho} + \frac{1}{\rho^2}\right).
    \end{align*}
\end{theorem}

\noindent In particular for fixed $\|\theta_0\|$ and $\rho$, a batch size $\Theta(d)$ suffices for a constant failure probability. We highlight the fact that the above result only uses $\eps$-sphericity along a 2-dimensional subspace. As detailed further in \Cref{sec:approximate-ellipticity-assumption}, one can expect this to hold for most subspaces of constant dimension (in the sense of the Haar measure over the Stiefel manifold) under very generic circumstances.

To prove \Cref{thm:isotropic-result}, we use a continuous-time argument. We start by by exhibiting idealized guarantees via an analysis of the gradient flow. Two features of the setting motivate this approach:  (1) we find ourselves interested in maximizing a different objective function than that which defines our available gradient queries; (2) we are primarily interested in understanding the starting conditions which enable learning under weak supervision. For both of these reasons, we find it difficult to apply the standard toolkit of fixed-point optimization theory. 

Instead, we take a polynomial approximation of the gradient flow. Formally, fix some $\tau > 0$, and consider the initial value problem defined by the vector field $-\nabla \loss : \R^d \to \R^d$ with initial position $\theta_0$. By the global extension of Picard-Lindel\"of, there exists a unique solution $[0, \tau] \to \R^d$, and we denote these dynamics $(\Theta_s)_{s \geq 0}$. Letting $\Psi(\cdot) = \cos(\cdot, \theta^*)$, Taylor theorem guarantees that within a small neighborhood of zero, \begin{align*}
    \Psi(\Theta_s) &\approx \Psi(\Theta_0) - \nabla\Psi(\Theta_0)^\top \nabla \loss(\Theta_0)s \\
        &\hspace{4em} + \Big(\nabla \loss(\Theta_0)^\top \nabla^2 \Psi(\Theta_0) \nabla \loss(\Theta_0) + \nabla \Psi(\Theta_0)^\top \nabla^2 \loss(\Theta_0) \nabla \loss(\Theta_0)\Big) \frac{s^2}{2}.
\end{align*}
Using approximate ellipticity, we obtain tight connections between the first derivative and the progress term $\rho$. Moreover, because the Hessians of cosine similarity and logistic loss are sufficiently well-behaved, we can control second-order errors long enough to make constant progress.

In order to extend to the stochastic result, we first show that a discretization of the gradient flow remains stable with short step sizes. That is, for discrete, full-information iterates $\bar{\theta}_{t+1} = \bar{\theta}_t - \eta\nabla \loss(\bar{\theta}_t)$, the gap $\|\Theta_{\eta t} - \bar{\theta}_t\|$ remains small. Again, because we only ask for constant improvement in the cosine similarity, this step size requirement remains reasonable under well-conditioned distributions. Second, we show that with sufficiently large batches, SGD iterates $\hat\theta_{t+1} = \hat\theta_t - \eta g_t$ closely approximate the discretization. Since $\Psi$ remains $2/r$-Lipschitz outside of the $r$-ball, we can show that $\Psi(\hat\theta_T)$ remains within a constant factor of $\Psi(\Theta_{\eta T})$ with high probability. We formalize these steps in \Cref{sec:full-results}.

\section{Empirical Simulations}
\label{sec:exp}

We now empirically confirm our theoretical findings. Across the training experiments in this section, we assume a standard Gaussian data distribution and take $\mu = \CN(0, I)$ in $d=100$ dimensions. We set $\theta^*$ to be the first basis vector, which always holds up to rotation.  

We first validate our main result \Cref{thm:isotropic-result} by showing that, for a fixed student, almost all randomly sampled teachers that are weaker by a constant gap suffice to improve the student. We find this unsurprising: randomly drawn teachers teachers have vanishing correlation with the student in directions orthogonal to $\theta^*$, thus satisfying the condition $\rho > 0$ (\Cref{eq:overview-signal-term}).

However, it is not obvious that randomly sampling teachers is realistic in practice, as student and teacher models typically arise from training with (1) similar algorithms on (2) similar data distributions, restricting the space of learnable models and potentially inducing correlations. Therefore, we next consider a ``shared pretraining'' setting where the student and teachers are trained from independent random initializations on the same data distribution, but the student is given strictly more resources---additional data and longer training---resulting in a stronger model. We observe that weak-to-strong gains persist even under this regime.

\subsection{Learning from Randomly Sampled Weaker Teachers}
\label{sec:exp:random}

\begin{figure}[t]
    \centering    \includegraphics[width=.8\textwidth]{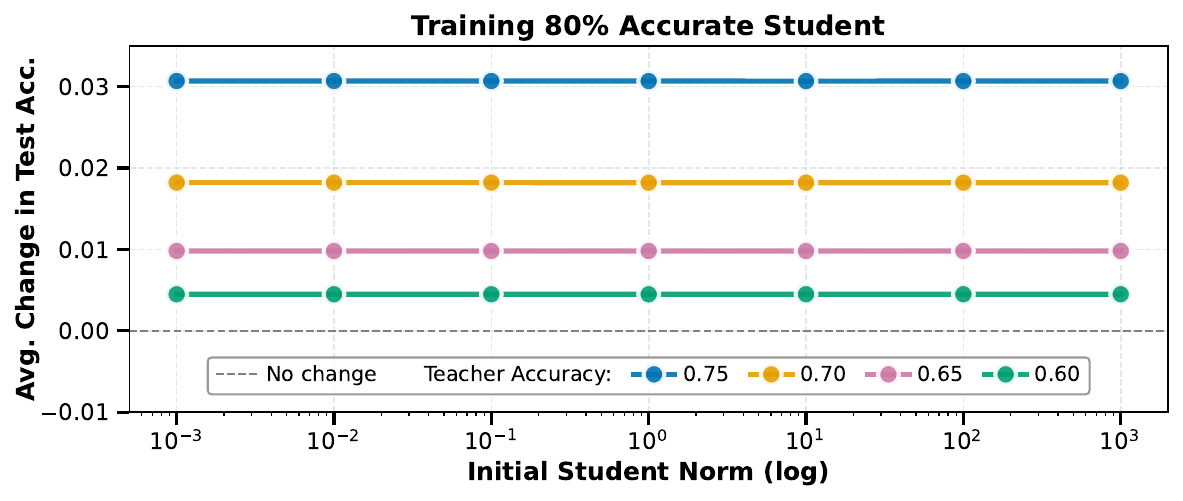}
    \caption{We finetune an 80\% accurate student linear model with weaker teachers of varying (lower) accuracy (different colors) and across various initial student norms. For each setting, we visualize the average gain in test accuracy across 100 different random teachers, training with early stopping. Consistent with our theoretical predictions, weaker teachers improve the student less, but all weak teachers yield non-trivial gains. Surprisingly, we observe that the student's initial norm does not affect the total gain, diverging from our analysis.}
    \label{fig:teacher_strength_norms}
\end{figure}

We fix a target initial student accuracy (equivalently, cosine similarity with $\theta^*$) and norm, then obtain the weights $\theta_0$ for the student model by sampling uniformly from the sphere satisfying these constraints. We fix a strictly lower target teacher accuracy, and similarly uniformly sample 100 unit-norm teachers that are weaker. For each teacher, we train student $\theta_0$ via SFT on the teacher's pseudo-labels (see Definition~\ref{def:sft}). We sweep learning rate $\eta$ in a fixed set (\texttt{np.logspace(-5, 1, 7)}), train with batch size 10,000, and employ early stopping. At each step of the training run, we measure the current $t$-th student iterate's (1) inner product $\langle \theta_t, \theta*\rangle$, (2) weight norm $||\theta_t||$, and cosine similarity with $\theta^*$. 

We report main results in Figure~\ref{fig:teaser}, showing trajectories from training 80\% accurate students of varying norms with 70\% accurate teachers. We observe three key takeaways. First, all weak teachers improve the student; every trajectory strictly exceeds the initial student's cosine similarity with $\theta^*$ at some point. Second, the total gain appears to be independent of the student's initial norm; trajectories peak at the same average gain\footnote{This differs from our analysis which predicts gains that decay with larger student norms; we view this as an interesting inquiry for future work. See Section~\ref{sec:outlook} for further discussion.}. Finally, learning dynamics depend heavily on the student's initial norm. Roughly speaking, we observe two distinct learning mechanisms: when the norm is high, $\rho_\ell$ is negative and $\rho_\upsilon$ dominates (Equation~\ref{eq:overview-signal-term}). In other words, the student improves by unlearning its own orthogonality with $\theta^*$; despite its inner product with $\theta^*$ decreasing, its norm decreases faster (purple line in Figure~\ref{fig:teaser}). In contrast, when the student is initially unconfident, $\rho_\ell$ is positive and dominates, yielding an update that is positively correlated with $\theta^*$. Hence the student learns from the teacher's knowledge of $\theta^*$, and we get gains in early steps where the inner product with $\theta^*$ grows faster than the norm (blue line).

Note that our results also highlight the importance of early stopping. As discussed earlier, one cannot expect to get vanishing loss from this weak-to-strong setup. After some steps of training, the gain peaks, and training for too long begins to actively hurt. See \Cref{fig:teacher_strength_norms} for a more extensive sweep of initial student norms and teacher accuracies; results corroborate our above discussion.

\begin{figure}[t]
    \centering
    \includegraphics[width=\textwidth]{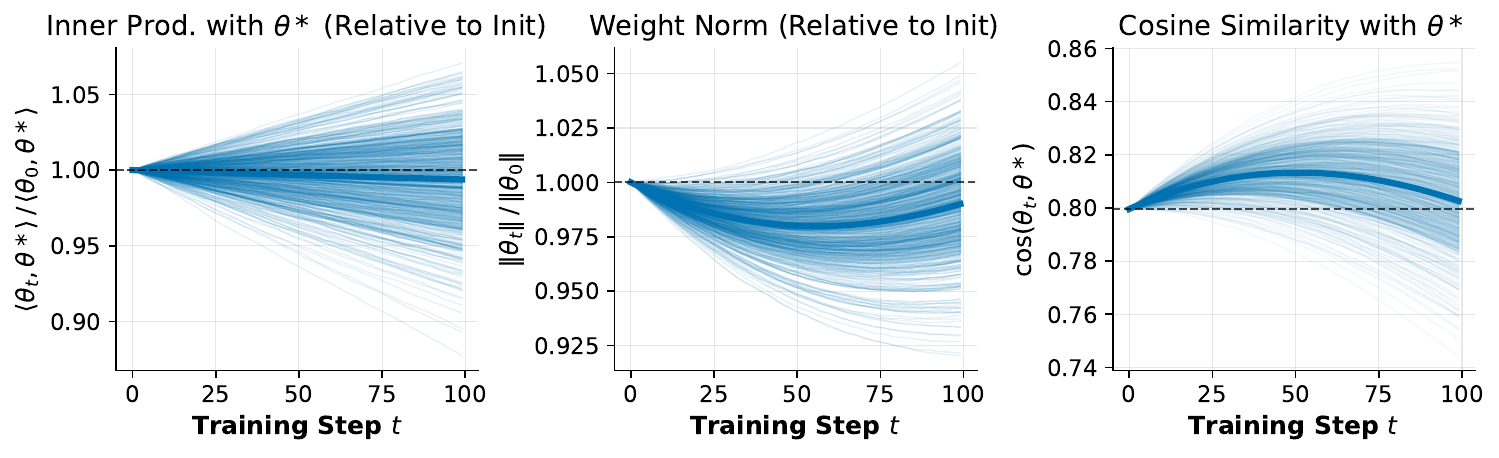}
    \caption{We first obtain a strong student linear model and 100 weaker teachers by pretraining independent random initializations on overlapping ground-truth data with varying compute, and then finetune the student with each weak teacher's pseudo-labels (separate runs, thin lines). Despite their shared origin, nearly all weak teachers suffice to improve the stronger student's performance. Thick lines and shading show mean $\pm$ std.}
    \label{fig:shared_pretrain}
\end{figure}

\subsection{Learning from Weaker Teachers with Similar Pretraining}
\label{sec:exp:shared}

We simulate shared pretraining by deriving both student and teachers from training on a small ground-truth dataset. We pretrain the student $\theta_0$ from random Gaussian initialization for 25 epochs on 5000 correctly-labeled points with learning rate $\eta=0.1$. We similarly obtain 100 weaker teachers via training (from independent random initializations) with the same data distribution and learning rate but $\sim$10x less total compute: 10 epochs on a strict 1000-point subset. At a high level, our setup is inspired by current language model development, where stronger models are obtained by scaling the pretraining data and compute. This process yields a student with $\sim$80\% accuracy and teachers with $\sim$63\% accuracy. Afterwards, with each teacher, we finetune $\theta_0$ using pseudo-labels from the teacher and visualize training trajectories with largely the same setup described earlier (Section~\ref{sec:exp:shared}). One deviation is we use a smaller batch size of 1000, chosen to stay true to our ``pretrain-then-finetune'' emulation. 
We report results in Figure~\ref{fig:shared_pretrain}. Overall, despite the shared pretraining, nearly all of the weak teachers remain sufficient to improve the stronger student.

\section{Outlook}
\label{sec:outlook}
As alluded to above, our results decay with larger student norms. Considering the local nature of our Taylor approximation arguments, this is unsurprising; even the simplified sketch in our introduction allows for the derivative of $t \mapsto \cos(\ell(t), \theta^*)$ at zero to decay with $1/\|\theta_0\|$. Nevertheless, our experiments suggest that, at least in the Gaussian setting, total progress should be nearly independent of this quantity. This phenomenon appears to evade the present analysis, and may be an interesting line for further inquiry.

Another natural question is whether we can extend these results beyond linear models. 
Indeed, because the gains we see arise from infinitesimal dynamics, we conjecture that there are natural conditions so that for more general kernels and/or learned representations, we can see a similar weak-to-strong generalization phenomenon.
However, we leave the details of this to future work.

\clearpage
\subsection*{Acknowledgments}
SG is supported by the NSF GRFP. DH was partially supported by the NSF Institute for Foundations of Machine Learning.

\bibliographystyle{alpha}
\bibliography{refs}

\clearpage
\tableofcontents

\clearpage
\appendix
\crefalias{section}{appendix}
\renewcommand{\theHsection}{A\arabic{section}}
\renewcommand{\theHtheorem}{appendix.\Alph{section}.\arabic{theorem}}
\renewcommand{\theHlemma}{appendix.\Alph{section}.\arabic{lemma}}
\renewcommand{\theHdefinition}{appendix.\Alph{section}.\arabic{definition}}
\renewcommand{\theHexample}

\section{Full Results}\label{sec:full-results}

We now state our results in full generality. Throughout, we consider a covariance $\Sigma \succ 0$. We take $\|\theta^*\|_\Sigma = \|\psi\|_\Sigma = 1$, and $\tilde\theta_0 := \theta_0/\|\theta_0\|_\Sigma \neq 0$. Here, the right analog of the first condition in \cref{thm:isotropic-result} turns out to be the twofold assumption that $\langle \theta^*, \theta_0 \rangle_{\Sigma^2} \geq 0$ and $\langle \theta^*, \theta_0\rangle_\Sigma \geq 0$, which we assume throughout this section. We again take $m_{\mu, \beta}(\theta_0)$ as written in \Cref{sec:sufficient-conditions}, and generalize the quantities of interest: 
\begin{align*}
    \rho_\ell &:= \frac{\lambdamin}{2} \E|\psi^\top\bm{x}| \cdot \cos_\Sigma(\Sigma \theta^*, \psi) - m_{\mu, \beta}(\theta_0) \cdot \langle \Sigma \theta^*, \tilde{\theta}_0\rangle_\Sigma -2\lambdamax\eps,\\
    \rho_\upsilon &:= m_{\mu, \beta}(\theta_0) \cdot \left\langle \Sigma \on{Proj}_{(\theta^*)^\perp}^\Sigma \tilde\theta_0, \tilde{\theta}_0 \right\rangle_\Sigma -\frac{\lambdamax}{2} \E|\psi^\top\bm{x}|\cdot |\cos_{\Sigma}(\Sigma \on{Proj}_{(\theta^*)^\perp}^\Sigma \theta_0,\psi)|-2 \lambdamax \eps.
\end{align*}
The sufficient condition for weak-to-strong guarantees now becomes \begin{align}\label{eq:main-signal-term}
    \rho &= (1-\cos_\Sigma^2(\theta^*, \theta_0))\rho_\ell + \cos_\Sigma(\theta^*, \theta_0)\rho_\upsilon >0.
\end{align} 
Under this hypothesis, we will see that it is possible to improve the cosine similarity (with respect to $\Sigma$) by \begin{align*}
    \cos(\hat{\theta}_T, \theta^*) - \cos(\theta_0, \theta^*) \geq C_{\lambdamax, \lambdamin} \cdot \left(\frac{\rho}{\|\theta_0\| + 1 + \eps}\right)^2,
\end{align*}
where $C_{\lambdamax, \lambdamin}$ depends only on the top and bottom eigenvalues of the covariance.
We state the precise guarantees in \cref{thm:stochastic-analysis}, and later refine the relevant bounds to match this more concise statement of \cref{thm:isotropic-result}.

In \cref{sec:sufficient-conditions}, we described how cancellation in $\rho$ suggests that the weak-to-strong phenomenon is plausible under only mild assumptions. While we don't get full cancellation in the non-isotropic case, one can observe that an approximate version of \cref{eq:overview-signal-term-cancellation} still holds. Naturally, this cancellation decays with the conditioning and with the similarity between $\theta^*$ and $\theta_0$ orthogonal to $\theta^*$ (with respect to the $\Sigma$ inner product). Concretely, under the assumptions above, it will hold that \begin{align*}
    \rho \geq \frac{\E |\psi^\top \bm{x}|}{2} &\left( \lambdamin \big(1-\cos_\Sigma^2(\theta^*, \theta_0)\big) \cos_\Sigma(\Sigma \theta^*, \psi) - \lambdamax \cos_\Sigma(\theta^*, \theta_0) \cdot |\cos_\Sigma(\Sigma \on{Proj}_{(\theta^*)^\perp}^\Sigma \theta_0, \psi)| \right)\\
    - &(\lambdamax-\lambdamin)m_{\mu, \beta}(\theta_0) \sum_{E \in \on{Eig}(\Sigma)} \left|\left\langle \on{Proj}_E \theta^*, \on{Proj}_E \tilde{\theta}_0 \right\rangle\right| - 4\lambdamax \eps,
\end{align*}
where $\on{Eig}(\Sigma)$ denotes the collection of eigenspaces of $\Sigma$. We prove this in \cref{lem:ill-conditioned-inevitable}.

\subsection{Gradient Flow Analysis}\label{sec:gradient-flow}

To start, we exhibit our primary application of approximate ellipticity. Roughly, this condition allows us some handle on how the logistic gradients evolve cosine similarity, and enables our analysis of the gradient flow.
Throughout, we also write $\Psi(\cdot) = \cos_\Sigma(\cdot, \theta^*)$ for brevity. 

\begin{lemma}\label{lem:first-order-lower-bound}
    Let $\bm{x} \sim \mu$ be $\eps$-elliptical along $\on{span}\{\psi, \theta_0\}$.
    Then provided $\rho > 0$, it holds that \begin{align*}
        \left\langle \nabla \Psi(\theta_0), - \nabla \loss(\theta_0)\right\rangle_I \geq \frac{\beta\rho}{\lambdamax^{1/2}\|\theta_0\|}.
    \end{align*}
\end{lemma}

\begin{proof}
    By \Cref{lem:cosine-grad-hess}, \begin{align*}
        \left\langle \nabla \Psi(\theta_0), - \nabla \loss(\theta_0)\right\rangle_I = \frac{1}{\|\theta_0\|_\Sigma} \left\langle \on{Proj}_{(\theta_0)^\perp}^\Sigma(\theta^*), -\nabla \loss(\theta_0)\right\rangle_\Sigma.
    \end{align*}
    Decomposing $-\nabla \loss(\theta_0)$ into it's orthogonal projection along $\theta^*$ (with respect to $\Sigma$), and letting $\tilde\theta_0 = \theta_0/\|\theta_0\|_\Sigma$, one obtains \begin{align}\label{eq:1}
        \left\langle \on{Proj}_{(\theta_0)^\perp}^\Sigma(\theta^*), -\nabla \loss(\theta_0)\right\rangle_\Sigma &= -\left\langle \theta^*, \on{Proj}_{(\theta_0)^\perp}^\Sigma \nabla \loss(\theta_0)\right\rangle_\Sigma\\
        &= \langle \tilde{\theta}_0, \nabla \loss(\theta_0)\rangle_{\Sigma} \langle \theta^*, \tilde{\theta}_0\rangle_\Sigma - \langle \theta^*, \nabla \loss(\theta_0)\rangle_\Sigma.\nonumber
    \end{align}
    Expanding $\nabla \loss(\theta_0)$ about $\theta^*$, we have \begin{align*}
        \langle \theta^*, \nabla \loss(\theta_0)\rangle_\Sigma = \langle \tilde{\theta}_0, \theta^*\rangle_\Sigma \langle \theta^*, \nabla \loss(\theta_0)\rangle_\Sigma + \left\langle \tilde{\theta}_0, \on{Proj}_{(\theta^*)^\perp}^\Sigma \nabla \loss(\theta_0)\right\rangle_\Sigma.
    \end{align*}
    Hence, \begin{align*}
        \left\langle \on{Proj}_{(\theta_0)^\perp}^\Sigma(\theta^*), -\nabla \loss(\theta_0)\right\rangle_\Sigma = (1-\cos_\Sigma(\theta_0, \theta^*)^2 )\langle \theta^*, - \nabla \loss(\theta_0)\rangle_\Sigma + \cos_\Sigma(\theta_0, \theta^*) \left\langle \tilde{\theta}_0, \on{Proj}_{(\theta^*)^\perp}^\Sigma \nabla \loss(\theta_0)\right\rangle_\Sigma.
    \end{align*}
    
    First, applying approximate ellipticity, we note that 
    \begin{align*}
        \E\left[ \ind(\psi^\top\bm{x} \geq 0) \cdot \langle \Sigma \theta^*, \bm{x}\rangle_I \right] &= \E\left[ \ind(\psi^\top\bm{x} \geq 0) \cdot \E\left[ \langle \Sigma \theta^*, \bm{x}\rangle_I \,\big|\, \psi^\top\bm{x}\right] \right] \\
        &\geq \left(\frac{(\Sigma \theta^*)^\top \Sigma \psi}{\psi^\top \Sigma \psi}\right) \E[\ind(\psi^\top \bm{x} \geq 0) \psi^\top \bm{x}] - \eps\|\Sigma \theta^*\|_{\Sigma}.
    \end{align*}
    Similarly, 
    \begin{align*}
        \E\left[ \sigma(\theta_0^\top \bm{x}) \langle \Sigma \theta^*, \bm{x}\rangle_I \right] &= \E\left[ \sigma(\theta_0^\top \bm{x}) \cdot \E\left[ \langle \Sigma \theta^*, \bm{x}\rangle_I \,\big|\, \theta_0^\top \bm{x} \right] \right]\\
        &\leq \left(\frac{(\Sigma \theta^*)^\top \Sigma \theta_0}{\theta_0^\top \Sigma \theta_0}\right) \E[\sigma(\theta_0^\top\bm{x}) \theta_0^\top \bm{x}] + \eps\|\Sigma \theta^*\|_{\Sigma}.
    \end{align*}
    Hence, using $\|\psi\|_\Sigma = 1$ and $\|\Sigma \theta^*\|_{\Sigma} \leq \lambdamax \|\theta^*\|_\Sigma = \lambdamax$, \begin{align*}
        \langle \theta^*, -\nabla \loss(\theta_0)\rangle_\Sigma &= \beta \E\left[ \ind(\psi^\top\bm{x} \geq 0) \cdot \langle \Sigma \theta^*, \bm{x}\rangle_I - \sigma(\theta_0^\top\bm{x})\cdot \langle \Sigma\theta^*, \bm{x}\rangle_I \right]\\
        &\overset{(i)}{\geq} \beta \langle \Sigma\theta^*, \psi\rangle_\Sigma \cdot \E\left[ \ind(\psi^\top\bm{x} \geq 0) \cdot \langle \psi, \bm{x}\rangle_I \right] - \beta \frac{\langle \Sigma \theta^*, \theta_0 \rangle_\Sigma}{\langle \theta_0, \theta_0\rangle_\Sigma} \cdot \E\left[ \sigma(\theta_0^\top\bm{x}) \cdot \langle \theta_0,\bm{x}\rangle_I \right]-2\beta\lambdamax\eps\\
        &\overset{(ii)}{=} \frac{\beta}{2} \langle \Sigma\theta^*, \psi\rangle_\Sigma \cdot \E|\psi^\top\bm{x}| - \beta \langle \Sigma \theta^*, \tilde{\theta}_0 \rangle_\Sigma \cdot \E\left[ \sigma(\theta_0^\top\bm{x}) \cdot \langle \tilde{\theta}_0,\bm{x}\rangle_I \right]-2\beta\lambdamax\eps\\
        &\overset{(iii)}{\geq} \frac{\beta \lambdamin}{2} \cos_\Sigma(\Sigma \theta^*, \psi) \cdot \E|\psi^\top\bm{x}| - \beta \langle \Sigma \theta^*, \tilde{\theta}_0 \rangle_\Sigma \cdot m_{\mu, \beta}(\theta_0)-2\beta\lambdamax\eps.
    \end{align*}
    In \emph{(i)}, we used non-negativity of $\sigma$ and the assumption $\langle \theta^*, \theta_0 \rangle_{\Sigma^2} \geq 0$. In \emph{(ii)} we used that $\E \bm{x} = 0$, which gives $\E \ind(\psi^\top\bm{x} \geq 0) \cdot \langle \psi, \bm{x}\rangle_I = \E|\psi^\top\bm{x}|/2$. In \emph{(iii)}, we used $\lambdamin\|\theta^*\|_\Sigma \leq \|\Sigma \theta^*\|_\Sigma$. Using a similar argument, \begin{align*}
        &\left\langle \on{Proj}_{(\theta^*)^\perp}^\Sigma \nabla \loss(\theta_0), \tilde\theta_0 \right\rangle_\Sigma = \left\langle \nabla \loss(\theta_0), \on{Proj}_{(\theta^*)^\perp}^\Sigma \tilde\theta_0 \right\rangle_\Sigma\\
        &\quad \quad \quad \quad = \beta\E\left[\sigma(\theta_0^\top\bm{x})\cdot \left\langle \Sigma \on{Proj}_{(\theta^*)^\perp}^\Sigma \tilde\theta_0, \bm{x}\right\rangle_I - \ind(\psi^\top\bm{x} \geq 0) \cdot \left\langle \Sigma \on{Proj}_{(\theta^*)^\perp}^\Sigma \tilde\theta_0, \bm{x}\right\rangle_I \right]\\
        &\quad \quad \quad \quad \geq \beta \left\langle \Sigma \on{Proj}_{(\theta^*)^\perp}^\Sigma \tilde\theta_0, \tilde{\theta}_0 \right\rangle_\Sigma \cdot m_{\mu, \beta}(\theta_0) - \frac{\beta}{2}\left|\left\langle \Sigma \on{Proj}_{(\theta^*)^\perp}^\Sigma \tilde\theta_0, \psi\right\rangle_\Sigma\right| \cdot \E|\psi^\top\bm{x}| - 2\beta\eps \|\Sigma \on{Proj}_{(\theta^*)^\perp}^{\Sigma} \tilde{\theta}_0 \|_{\Sigma} \\
        &\quad \quad \quad \quad \overset{(iv)}{\geq} \beta \left\langle \Sigma \on{Proj}_{(\theta^*)^\perp}^\Sigma \tilde\theta_0, \tilde{\theta}_0 \right\rangle_\Sigma\cdot m_{\mu, \beta}(\theta_0) -\frac{\beta \lambdamax}{2} |\cos_{\Sigma}(\Sigma \on{Proj}_{(\theta^*)^\perp}^\Sigma \theta_0,\psi)|\cdot \E|\psi^\top\bm{x}| - 2\beta\lambdamax\eps.
    \end{align*}
    In \emph{(iv)}, we used $\|\Sigma \on{Proj}_{(\theta^*)^\perp}^\Sigma \tilde\theta_0\|_\Sigma \leq \lambdamax$. To conclude, use the assumption $\langle \theta^*, \theta_0\rangle_\Sigma \geq 0$ and $\|\theta_0\|_\Sigma \leq \lambdamax^{1/2}\|\theta_0\|$.
\end{proof}

\begin{lemma}\label{lem:ill-conditioned-inevitable}
    Let $\bm{x} \sim \mu$ be $\eps$-elliptical along $\on{span}\{\psi, \theta_0\}$.
    Then provided $\rho > 0$, it holds that \begin{align*}
        \rho \geq \frac{\E |\psi^\top \bm{x}|}{2} &\left( \lambdamin \big(1-\cos_\Sigma^2(\theta^*, \theta_0)\big) \cos_\Sigma(\Sigma \theta^*, \psi) - \lambdamax \cos_\Sigma(\theta^*, \theta_0) \cdot |\cos_\Sigma(\Sigma \on{Proj}_{(\theta^*)^\perp}^\Sigma \theta_0, \psi)| \right)\\
    - &(\lambdamax-\lambdamin)m_{\mu, \beta}(\theta_0) \sum_{E \in \on{Eig}(\Sigma)} \left|\left\langle \on{Proj}_E \theta^*, \on{Proj}_E \tilde{\theta}_0 \right\rangle\right| - 4\lambdamax \eps.
    \end{align*}
\end{lemma}

\begin{proof}
    Take \begin{align*}
        A &:= \frac{\E |\psi^\top \bm{x}|}{2} \left(\lambdamin \big(1-\cos_\Sigma^2(\theta^*, \theta_0)\big) \cos_\Sigma(\Sigma \theta^*, \psi) - \lambdamax \cos_\Sigma(\theta^*, \theta_0) |\cos_\Sigma(\Sigma \on{Proj}_{(\theta^*)^\perp}^\Sigma \theta_0, \psi)| \right),\\
        C &:= m_{\mu, \beta}(\theta_0) \left( \big(1-\cos_\Sigma^2(\theta^*, \theta_0)\big) \langle \Sigma \theta^*, \tilde{\theta}_0\rangle_\Sigma - \cos_\Sigma(\theta^*, \theta_0) \left\langle \Sigma \on{Proj}_{(\theta^*)^\perp}^\Sigma \tilde{\theta}_0, \tilde{\theta}_0 \right\rangle_\Sigma \right),
    \end{align*}
    and note that by rearranging terms, \begin{align*}
        \rho \geq A -C - 4\lambdamax \eps.
    \end{align*}
    Writing $\on{Proj}_{(\theta^*)^\perp}^\Sigma = I - \langle \theta^*, \cdot\rangle_\Sigma \theta^*$, we observe
    \begin{align*}
        C &= m_{\mu, \beta}(\theta_0) \left( \langle \theta^*, \tilde{\theta}_0 \rangle_{\Sigma^2} - \langle \theta^*, \tilde{\theta}_0 \rangle_\Sigma \langle \tilde{\theta}_0, \tilde{\theta}_0\rangle_{\Sigma^2} \right)\\
        &= m_{\mu, \beta}(\theta_0) \cdot (\Sigma^{1/2} \theta^*)^\top \left( \Sigma - (\Sigma^{1/2} \tilde{\theta}_0)^\top \Sigma (\Sigma^{1/2} \tilde{\theta}_0) I \right) (\Sigma^{1/2} \tilde{\theta}_0)\\
        &= m_{\mu, \beta}(\theta_0) \sum_{E \in \on{Eig}(\Sigma)} \left( \lambda_E - (\Sigma^{1/2} \tilde{\theta}_0)^\top \Sigma (\Sigma^{1/2} \tilde{\theta}_0)\right) {\theta^*}^\top \Sigma^{1/2} \on{Proj}_E \Sigma^{1/2} \tilde{\theta}_0.
    \end{align*}
    Since $\|\Sigma^{1/2} \tilde{\theta}_0\|_I = 1$, \begin{align*}
        \lambdamin \leq (\Sigma^{1/2} \tilde{\theta}_0)^\top \Sigma (\Sigma^{1/2} \tilde{\theta}_0) \leq \lambdamax.
    \end{align*} Moreover, since the eigenspace projection commutes with $\Sigma$, \begin{align*}
        |C| \leq (\lambdamax - \lambdamin) \cdot m_{\mu, \beta}(\theta_0) \sum_{E \in \on{Eig}(\Sigma)} \left| \left\langle \on{Proj}_E \theta^*, \on{Proj}_E \tilde{\theta}_0 \right\rangle \right|.
    \end{align*}
    The statement follows.
\end{proof}

\begin{lemma}\label{lem:first-order-upper-bound}
    Let $\bm{x} \sim \mu$ be $\eps$-elliptical along $\on{span}\{\psi, \theta_0\}$.
    Then provided $\rho > 0$, it holds that \begin{align*}
        \|\nabla \loss(\theta_0)\| \leq \beta\lambdamax^{1/2}\left(\frac{3}{2} + 2\eps\right).
    \end{align*}
\end{lemma}

\begin{proof}
    It suffices to bound $|\langle u, \nabla \loss(\theta_0)\rangle_I|$ for $\|u\|=1$. As above, condition and apply approximate ellipticity. 
    \begin{align*}
        |\langle u, \nabla \loss(\theta_0)\rangle_I| & \leq \beta \left| \E \left[ \ind(\psi^\top\bm{x} \geq 0) \cdot \E\left[ \langle u, \bm{x}\rangle_I \,\Big|\, \psi^\top\bm{x} \right] \right]\right| + \beta \left| \E \left[ \sigma(\theta_0^\top\bm{x}) \cdot \E\left[ \langle u, \bm{x}\rangle_I \,\Big|\, \theta_0^\top\bm{x} \right] \right]\right|\\
        &\leq \frac{\beta}{2} \left| \langle u, \psi\rangle_\Sigma \right| \cdot \E|\psi^\top \bm{x}| + \beta \frac{|\langle u, \theta_0\rangle_\Sigma|}{\|\theta_0\|_\Sigma} \cdot m_{\mu, \beta}(\theta_0) + 2\beta\eps \|u\|_{\Sigma}\\
        &\leq \frac{\beta}{2} \lambdamax^{1/2} \E|\psi^\top\bm{x}| + \beta \lambdamax^{1/2} m_{\mu, \beta}(\theta_0) + 2\beta\lambdamax^{1/2} \eps.
    \end{align*}
    The last line is by $\|u\|_\Sigma \leq \lambdamax^{1/2}\|u\|$. To conclude the proof, recall that $m_{\mu, \beta} \leq 1$ and $\E |\psi^\top\bm{x}| \leq (\E (\psi^\top\bm{x})^2)^{1/2} = \|\psi\|_{\Sigma} = 1$.
\end{proof}

\begin{lemma}\label{lem:rho-upper-bound}
    Let $\bm{x} \sim \mu$ be $\eps$-elliptical along $\on{span}\{\psi, \theta_0\}$.
    Then provided $\rho > 0$, it holds that \begin{align*}
       \rho \leq \frac{3}{2}\lambdamax.
    \end{align*}
\end{lemma}

\begin{proof}
    Recalling the definitions from above, \begin{align*}
        \rho_\ell \leq \frac{\lambdamin}{2}.
    \end{align*}
    Also, \begin{align*}
        \rho_\upsilon \leq \left| \left\langle \Sigma \on{Proj}_{(\theta^*)^\perp}^\Sigma \tilde{\theta}_0, \tilde{\theta}_0\right\rangle_\Sigma \right|
        \leq \|\Sigma \|_\Sigma \left\| \on{Proj}_{(\theta^*)^\perp}^\Sigma \tilde{\theta}_0\right\|_\Sigma \leq \lambdamax.
    \end{align*}
    The claim follows by definition of $\rho$.
\end{proof}

\begin{lemma}\label{lem:grad-norm-lower-bound}
    Let $\bm{x} \sim \mu$ be $\eps$-elliptical along $\on{span}\{\psi, \theta_0\}$.
    Then provided $\rho > 0$, it holds that \begin{align*}
       \|\nabla \loss(\theta_0)\| \geq \frac{\beta\lambdamin^{1/2}\rho}{\lambdamax}.
    \end{align*}
\end{lemma}

\begin{proof}
    By \cref{lem:first-order-lower-bound}, \begin{align*}
        \|\nabla\loss(\theta_0)\| \geq \frac{\beta\rho}{\lambdamax^{1/2} \|\theta_0\| \|\nabla\Psi(\theta_0)\|}.
    \end{align*}
    Next, note that $\|\theta_0\|_\Sigma \geq \lambdamin^{1/2}\|\theta_0\|$, and by \cref{lem:cosine-grad-hess}, \begin{align*}
        \left\|\Sigma \on{Proj}_{(\theta_0)^\perp}^\Sigma(\theta^*)\right\|^2 &= {\theta^*}^\top (\on{Proj}_{(\theta_0)^\perp}^\Sigma)^\top \Sigma^2 (\on{Proj}_{(\theta_0)^\perp}^\Sigma) \theta^*\\
        &\leq \lambdamax \left\|\on{Proj}_{(\theta_0)^\perp}^\Sigma (\theta^*)\right\|_\Sigma^2\\
        &\leq \lambdamax.
    \end{align*}
    Hence, \begin{align*}
        \|\nabla \Psi(\theta_0)\| = \frac{1}{\|\theta_0\|_\Sigma} \left\|\Sigma \on{Proj}_{(\theta_0)^\perp}^\Sigma (\theta^*)\right\| \leq \frac{\lambdamax^{1/2}}{\lambdamin^{1/2}\|\theta_0\|}.
    \end{align*}
    And hence, \begin{align*}
        \|\nabla\loss(\theta_0)\| \geq \frac{\beta \lambdamin^{1/2}\rho}{\lambdamax}.
    \end{align*}
\end{proof}

We may now state the idealized gradient flow guarantee For technical reasons, we define a scaling quantity \begin{align*}
    \gamma &:= \frac{\lambdamax\beta^2}{\lambdamax\beta^2 + 4},
\end{align*} which will be dispensed with when necessary. Fix $\tau > 0$ to be chosen later, and consider the differentiable vector field $-\nabla \loss : \R^d \to \R^d$ with initial position $\theta_0$. By the global extension of Picard-Lindel\"of, there exists is a unique solution $[0, \tau] \to \R^d$ to this initial value problem, and we denote these dynamics $(\Theta_s)_{s \geq 0}$. So we have $\Theta_0 = \theta_0$ and velocity vector $\dot \Theta_s = -\nabla \loss(\Theta_s)$. 

\begin{lemma}\label{lem:gradient-flow}
    Let $\bm{x} \sim \mu$ be $\eps$-elliptical along $\on{span}\{\psi, \theta_0\}$. Suppose $\rho > 0$. Then by taking \begin{align*}
        \tau := &\frac{4}{\lambdamax\beta^2} \log\left(\frac{\gamma\|\theta_0\|}{\|\nabla \loss(\theta_0)\|} + 1\right) \,\,\land\,\, \frac{\beta\rho \gamma^2 \|\theta_0\|}{\lambdamax^{1/2} \left(\frac{2}{\sqrt{3}}\frac{\lambdamax}{\lambdamin} + \frac{\beta^2}{4}\frac{\lambdamax^{3/2}}{\lambdamin^{1/2}}\right) \left(\|\nabla \loss(\theta_0)\| + \gamma\|\theta_0\|\right)^2},
    \end{align*}
    it holds that \begin{align}\label{eq:gradient-flow-advantage}
        \Psi(\Theta_\tau) - \Psi(\Theta_0) &\geq \frac{2\rho}{\lambdamax^{3/2} \beta \|\theta_0\|} \log\left(\frac{\gamma \|\theta_0\|}{\|\nabla \loss(\theta_0)\|} + 1\right) \\
        &\land \frac{\beta^2 \rho^2 \gamma^2}{2\lambdamax \left(\frac{2}{\sqrt{3}} \frac{\lambdamax}{\lambdamin} + \frac{\beta^2}{4} \frac{\lambdamax^{3/2}}{\lambdamin^{1/2}} \right)\left(\|\nabla\loss(\theta_0)\| + \gamma \|\theta_0\|\right)^2}.\nonumber
    \end{align}
\end{lemma}

\begin{proof}
    We use a Taylor approximation. One computes $D_s(\Psi \circ \Theta_s) = -\nabla \Psi(\Theta_s)^\top \nabla \loss(\Theta_s)$ and ${\ddot \Theta_s} = -\nabla^2 \loss(\Theta_s) {\dot \Theta_s}$. With another application of chain rule, \begin{align*}
        D_s^2(\Psi \circ \Theta_s) &= -\langle D_s \nabla \Psi(\Theta_s), \nabla \loss(\Theta_s)\rangle - \langle \nabla \Psi(\Theta_s), D_s \nabla \loss(\Theta_s)\rangle\\
        &= \nabla \loss(\Theta_s)^\top \nabla^2 \Psi(\Theta_s) \nabla \loss(\Theta_s) + \nabla \Psi(\Theta_s)^\top \nabla^2 \loss(\Theta_s) \nabla \loss(\Theta_s).
    \end{align*}
    Fixing $t \in (0, \infty)$, there exists $z \in (0, t)$ so that \begin{align*}
        \Psi(\Theta_t) &= \Psi(\Theta_0) - \nabla\Psi(\Theta_0)^\top \nabla \loss(\Theta_0)t \\
        &+ \Big(\nabla \loss(\Theta_z)^\top \nabla^2 \Psi(\Theta_z) \nabla \loss(\Theta_z) + \nabla \Psi(\Theta_z)^\top \nabla^2 \loss(\Theta_z) \nabla \loss(\Theta_z)\Big) \frac{t^2}{2}.
    \end{align*}
    By \Cref{lem:first-order-lower-bound}, $- \nabla\Psi(\Theta_0)^\top \nabla \loss(\Theta_0) \geq \beta\rho/\lambdamax^{1/2}\|\theta_0\|$, so we turn to second-order bounds. Let $s \in (0, t)$. By \cref{lem:cosine-grad-hess}, \begin{align*}
        \|\nabla \Psi(\Theta_s)\| &= \frac{1}{\|\Theta_s\|_\Sigma} \|\Sigma \Proj_{(\theta_0)^\perp}^\Sigma(\theta^*)\|\\
        &= \frac{1}{\|\Theta_s\|_\Sigma} \|\Sigma^{1/2} \Proj_{(\theta_0)^\perp}^\Sigma(\theta^*)\|_\Sigma\\
        &\leq \frac{1}{\|\Theta_s\|_\Sigma}\|\Sigma^{1/2}\|_\Sigma \|\theta^*\|_\Sigma \leq \frac{\lambdamax^{1/2}}{\lambdamin^{1/2} \|\Theta_s\|},
    \end{align*} since $\|\Sigma^{1/2}\|_\Sigma = \lambdamax^{1/2}$. Using $\nabla^2 \loss \preceq \lambdamax\beta^2/4$, as well as \Cref{lem:hessian-top-eigenvalue} and the estimate above, \begin{align*}
        \left| \nabla \loss(\Theta_s)^\top \nabla^2 \Psi(\Theta_s) \nabla \loss(\Theta_s) + \nabla \Psi(\Theta_s)^\top \nabla^2 \loss(\Theta_s) \nabla \loss(\Theta_s)\right|\\
        \leq \frac{2}{\sqrt{3}}\frac{\lambdamax}{\lambdamin} \frac{\|\nabla \loss(\Theta_s)\|^2}{\|\Theta_s\|^2} + \frac{\beta^2}{4}\frac{\lambdamax^{3/2}}{\lambdamin^{1/2}} \frac{\|\nabla \loss(\Theta_s)\|}{\|\Theta_s\|}.
    \end{align*}
    
    From here, we use smoothness of the loss to bound drift of the iterates. We have \begin{align*}
        \Theta_s - \Theta_0 = - \int_0^s \nabla \loss(\Theta_a) \,da.
    \end{align*}
    By triangle inequality, \begin{align*}
        \|\Theta_s - \Theta_0\| \leq s\|\nabla\loss(\Theta_0)\| + \int_0^s \frac{\lambdamax\beta^2}{4} \|\Theta_a - \Theta_0\| \,da.
    \end{align*}
    By Gr\"onwall, we have  \begin{align*}
        \|\Theta_s - \Theta_0\| &\leq s\|\nabla \loss(\Theta_0)\| + \int_0^s \frac{a\lambdamax\beta^2}{4} \|\nabla \loss(\Theta_0)\| \exp\left(\frac{\lambdamax\beta^2(s-a)}{4}\right)da\\
        &= \|\nabla \loss(\Theta_0) \|\left(s + \frac{\lambdamax\beta^2}{4} e^{\lambdamax\beta^2 s/4} \left(\frac{1-e^{-\lambdamax\beta^2s/4}\left(\frac{\lambdamax\beta^2 s}{4} + 1\right)}{\lambdamax^2\beta^4/16}\right)\right)\\
        &= \frac{4}{\lambdamax\beta^2} \|\nabla \loss(\Theta_0)\|\left(e^{\lambdamax\beta^2 s/4} - 1\right).
    \end{align*}
    
    Take $\gamma \in (0, 1)$ to be chosen later. Take \begin{align*}
        \tau_1 := \frac{4}{\lambdamax\beta^2} \log\left(\frac{\lambdamax\beta^2(1-\gamma)}{4} \cdot \frac{\|\Theta_0\|}{\|\nabla \loss(\Theta_0)\|} + 1\right).
    \end{align*}
    Note that if $s \leq \tau_1$, then $\|\Theta_s - \Theta_0\| \leq (1-\gamma)\|\Theta_0\|$ and $\|\Theta_s\| \geq \gamma \|\Theta_0\|$. By smoothness, we also have $\|\nabla \loss(\Theta_s)\| \leq \|\nabla \loss(\Theta_0)\| + (\lambdamax\beta^2/4)(1-\gamma) \|\Theta_0\|$. 
    Updating our second-order bound, \begin{align*}
        &\left| \nabla \loss(\Theta_s)^\top \nabla^2 \Psi(\Theta_s) \nabla \loss(\Theta_s) + \nabla \Psi(\Theta_s)^\top \nabla^2 \loss(\Theta_s) \nabla \loss(\Theta_s)\right|\\
        &\quad \quad \leq \frac{2}{\sqrt{3}}\frac{\lambdamax}{\lambdamin} \frac{\left(\|\nabla \loss(\Theta_0)\| + \frac{\lambdamax\beta^2(1-\gamma)}{4} \|\Theta_0\|\right)^2}{\gamma^2 \|\Theta_0\|^2} + \frac{\beta^2}{4}\frac{\lambdamax^{3/2}}{\lambdamin^{1/2}} \frac{\left(\|\nabla \loss(\Theta_0)\| + \frac{\lambdamax\beta^2(1-\gamma)}{4} \|\Theta_0\|\right)}{\gamma \|\Theta_0\|}.
    \end{align*}
    As long as $\gamma \leq \lambdamax\beta^2/(\lambdamax\beta^2 + 4)$, then \begin{align*}
        \frac{\lambdamax\beta^2(1-\gamma)}{4} \|\Theta_0\| = \frac{\lambdamax\beta^2(1-\gamma)}{4\gamma} \cdot \gamma \|\Theta_0\| \geq \gamma \|\Theta_0\|,
    \end{align*}
    and hence
    \begin{align*}
        \frac{\|\nabla \loss(\Theta_0)\| + \frac{\lambdamax\beta^2(1-\gamma)}{4}\|\Theta_0\|}{\gamma \|\Theta_0\|} \geq 1.
    \end{align*}
    That is, \begin{align*}
        &\left| \nabla \loss(\Theta_s)^\top \nabla^2 \Psi(\Theta_s) \nabla \loss(\Theta_s) + \nabla \Psi(\Theta_s)^\top \nabla^2 \loss(\Theta_s) \nabla \loss(\Theta_s)\right|\\
        &\quad \quad \quad \quad \quad \leq \left(\frac{2}{\sqrt{3}}\frac{\lambdamax}{\lambdamin} + \frac{\beta^2}{4}\frac{\lambdamax^{3/2}}{\lambdamin^{1/2}}\right) \left(\frac{\|\nabla \loss(\Theta_0)\| + \frac{\lambdamax\beta^2(1-\gamma)}{4}\|\Theta_0\|}{\gamma \|\Theta_0\|}\right)^2.
    \end{align*}
    
    For $t \in (0, \tau_1]$, we obtain the updated Taylor bound: \begin{align*}
        &\Psi(\Theta_t) - \Psi(\Theta_0) \geq \frac{\beta \rho}{\lambdamax^{1/2} \|\theta_0\|} t - \left(\frac{2}{\sqrt{3}}\frac{\lambdamax}{\lambdamin} + \frac{\beta^2}{4}\frac{\lambdamax^{3/2}}{\lambdamin^{1/2}}\right) \left(\frac{\|\nabla \loss(\Theta_0)\| + \frac{\lambdamax\beta^2(1-\gamma)}{4}\|\Theta_0\|}{\gamma \|\Theta_0\|}\right)^2 \frac{t^2}{2}.
    \end{align*}
    Maximizing the quadratic over all of $\R$, the right-hand expression has a maximum at \begin{align*}
        \tau_2 := \frac{\beta\rho \cdot \gamma^2 \|\theta_0\|}{\lambdamax ^{1/2}\left(\frac{2}{\sqrt{3}}\frac{\lambdamax}{\lambdamin} + \frac{\beta^2}{4}\frac{\lambdamax^{3/2}}{\lambdamin^{1/2}}\right) \left(\|\nabla \loss(\theta_0)\| + \frac{\lambdamax\beta^2}{4} (1-\gamma)\|\theta_0\|\right)^2}.
    \end{align*}
    If $\tau_2 \leq \tau_1$, then \begin{align*}
        \Psi(\Theta_{\tau_2}) - \Psi(\Theta_0) &\geq \frac{\beta^2 \rho^2 \gamma^2}{2\lambdamax \left(\frac{2}{\sqrt{3}} \frac{\lambdamax}{\lambdamin} + \frac{\beta^2}{4} \frac{\lambdamax^{3/2}}{\lambdamin^{1/2}} \right)\left(\|\nabla\loss(\theta_0)\| + \gamma \|\theta_0\|\right)^2}.
    \end{align*}
    If instead $\tau_2 > \tau_1$, this implies
    \begin{align*}
        \tau_1 \cdot \left(\frac{2}{\sqrt{3}}\frac{\lambdamax}{\lambdamin} + \frac{\beta^2}{4}\frac{\lambdamax^{3/2}}{\lambdamin^{1/2}}\right) \left(\frac{\|\nabla \loss(\Theta_0)\| + \frac{\lambdamax\beta^2(1-\gamma)}{4}\|\Theta_0\|}{\gamma \|\Theta_0\|}\right)^2 \leq \frac{\beta\rho}{\lambdamax^{1/2} \|\theta_0\|}.
    \end{align*}
    And in particular, \begin{align*}
        \Psi(\Theta_{\tau_1}) - \Psi(\Theta_0) \geq \frac{\tau_1}{2} \frac{\beta\rho}{\lambdamax^{1/2} \|\Theta_0\|} = \frac{2\rho}{\lambdamax^{3/2} \beta \|\Theta_0\|} \cdot \log\left(\frac{\lambdamax\beta^2(1-\gamma)}{4} \cdot \frac{\|\Theta_0\|}{\|\nabla \loss(\Theta_0)\|} + 1\right).
    \end{align*}
    One notes that for $\gamma = \lambdamax\beta^2/(\lambdamax\beta^2+4)$, \begin{align*}
        \frac{\lambdamax\beta^2(1 -\gamma)}{4} = \gamma.
    \end{align*}
    Hence, taking $\tau = \min\{\tau_1, \tau_2\}$ one obtains the stated result.
\end{proof}

\subsection{Stochastic Analysis}\label{sec:stochastic-analysis}

The reduction from discrete to continuous-time guarantees follows a two-step argument: approximate the gradient flow with discrete full-information updates, then show that SGD approximates these updates with little error when steps are sufficiently small and batches sufficiently large. The former step will follow a classic argument; discrete gradient methods are well-known to approximate initial-value solutions under mild assumptions and small step sizes (see \emph{e.g.} the discussion of Euler's method in \cite[Chapter~12]{suli2003numerical}). While the necessary step sizes can be quite small (and hence the number of steps quite large), this will suffices for our purposes as we only ask for constant progress. Proceeding, we will denote by $\bar{\theta}_0, \cdots, \bar{\theta}_T$ the deterministic iterates on population gradients, where $\bar{\theta}_0 = \theta_0$ and $\bar{\theta}_{t+1} = \bar{\theta}_t - \eta\nabla \loss(\bar{\theta}_t)$. 

\begin{lemma}\label{lem:euler-approximation}
    Let $T \in \N$. Taking $\eta \leq \tau_1/T$, where 
    \begin{align*}
        \tau_1 := \frac{4}{\lambdamax\beta^2} \log\left(\frac{\gamma \|\theta_0\|}{\|\nabla \loss(\theta_0)\|} + 1\right),
    \end{align*}
    we have that for all $t \in \{0, \ldots, T\}$, \begin{align*}
        \|\Theta_{\eta t} - \bar{\theta}_t\| &\leq \eta \left(\frac{\gamma}{2}\|\theta_0\| + \frac{\gamma^2}{2}\frac{\|\theta_0\|^2}{\|\nabla \loss(\theta_0)\|}\right).
    \end{align*} 
\end{lemma}

\begin{proof}
    To start, take \begin{align*}
        \omega := \max_{0 \leq t \leq T-1} \left\| \nabla \loss(\Theta_{\eta t}) - \frac{\Theta_{\eta t} - \Theta_{\eta (t+1)}}{\eta}\right\|,
    \end{align*}
    and consider the decomposition \begin{align*}
        \Theta_{\eta(t+1)} - \bar{\theta}_{t+1} = \Big( \Theta_{\eta t} - \bar{\theta}_t \Big) + \eta \Big( \nabla \loss(\bar{\theta}_t) - \nabla \loss(\Theta_{\eta t})\Big) + \Big( (\Theta_{\eta(t+1)} - \Theta_{\eta t}) + \eta \nabla \loss(\Theta_{\eta t})\Big).
    \end{align*}
    By smoothness of $\loss$, we have \begin{align*}
        \|\Theta_{\eta(t+1)} - \bar{\theta}_{t+1}\| \leq (1+\lambdamax\beta^2\eta/4) \|\Theta_{\eta t} - \bar{\theta}_{t}\| + \eta \omega.
    \end{align*}
    By induction, one has \begin{align*}
        \|\Theta_{\eta(t+1)} - \bar{\theta}_{t+1}\| \leq \frac{4\omega}{\lambdamax\beta^2} \left( \left(1 + \frac{\lambdamax\beta^2\eta}{4}\right)^t - 1\right).
    \end{align*}
    To bound $\omega$, we use the integral form of Taylor. \begin{align*}
        \Big\| \eta \nabla \loss(\Theta_{\eta t}) - (\Theta_{\eta t} - \Theta_{\eta (t+1)}) \Big\| &= \left\| \int_{\eta t}^{\eta(t+1)} \nabla^2 \loss(\Theta_s) \nabla \loss(\Theta_s) ( \eta(t+1)-s) ds \right\|\\
        &\leq \frac{\lambdamax\beta^2\eta^2}{8} \, \sup_{s \in [\eta t, \eta(t+1)]} \|\nabla \loss(\Theta_s)\|.
    \end{align*}
    To conclude, recall from the proof of \cref{lem:gradient-flow} that by choice of $\tau_1$, we have $\|\nabla \loss(\Theta_s)\| \leq \|\nabla \loss(\theta_0)\| + \frac{\lambdamax\beta^2(1-\gamma)}{4}\|\theta_0\|$. Hence, for $0 \leq t \leq T-1$, \begin{align*}
        \|\Theta_{\eta (t+1)} - \bar{\theta}_{t+1}\| &\leq \frac{4\omega}{\lambdamax\beta^2}(e^{\lambdamax\beta^2\eta t/4} - 1)\\
        &\leq \frac{\eta}{2} \left( e^{\lambdamax\beta^2 \tau_1/4}-1\right)\left(\|\nabla \loss(\theta_0)\| + \frac{\lambdamax\beta^2(1-\gamma)}{4}\|\theta_0\|\right)\\
        &\leq \eta \left(\frac{\lambdamax\beta^2(1-\gamma)}{8} \frac{\|\theta_0\|}{\|\nabla \loss(\theta_0)\|}\right)\left(\|\nabla \loss(\theta_0)\| + \frac{\lambdamax\beta^2(1-\gamma)}{4}\|\theta_0\|\right).
    \end{align*}
    The choice of $\gamma$ gives the lemma as stated.
\end{proof}

From here, the stochastic argument uses only basic martingale tools obtain concentration on the gradient estimates. We believe that under stronger assumptions on the tails of $\mu$, a more practical batch size would suffice. Recall from \cref{sec:introduction} that we denote by $\hat{\theta}_0, \cdots, \hat{\theta}_T$ the stochastic process induced by SGD. For brevity, we will denote the minimum on the right-hand side of \cref{eq:gradient-flow-advantage} by: \begin{align*}
    \alpha_* := \frac{2\rho}{\lambdamax^{3/2} \beta \|\theta_0\|} \log\left(\frac{\gamma \|\theta_0\|}{\|\nabla \loss(\theta_0)\|} + 1\right) \land \frac{\beta^2 \rho^2 \gamma^2}{2\lambdamax \left(\frac{2}{\sqrt{3}} \frac{\lambdamax}{\lambdamin} + \frac{\beta^2}{4} \frac{\lambdamax^{3/2}}{\lambdamin^{1/2}} \right)\left(\|\nabla\loss(\theta_0)\| + \gamma \|\theta_0\|\right)^2}.
\end{align*}
We are now ready to state the stochastic result in full generality.

\begin{theorem}\label{thm:stochastic-analysis}
    Let $\mu$ be $\eps$-elliptical along $\on{span}\{\psi, \theta_0\}$. Suppose $\rho > 0$. Take \begin{align}\label{eq:tau-constraints}
        \tau := &\frac{4}{\lambdamax\beta^2} \log\left(\frac{\gamma\|\theta_0\|}{\|\nabla \loss(\theta_0)\|} + 1\right) \,\,\land\,\, \frac{\beta\rho \gamma^2 \|\theta_0\|}{\lambdamax^{1/2} \left(\frac{2}{\sqrt{3}}\frac{\lambdamax}{\lambdamin} + \frac{\beta^2}{4}\frac{\lambdamax^{3/2}}{\lambdamin^{1/2}}\right) \left(\|\nabla \loss(\theta_0)\| + \gamma\|\theta_0\|\right)^2},
    \end{align}
    and $\eta = \tau/T$ for integer $T$ such that \begin{align*}
        \delta := \frac{\gamma \|\theta_0\| \lambdamin^{1/2}\alpha_*}{8\lambdamax^{1/2}}, \quad \quad \eta \leq \delta \cdot\left(\frac{\gamma}{2}\|\theta_0\| + \frac{\gamma^2}{2}\frac{\|\theta_0\|^2}{\|\nabla \loss(\theta_0)\|}\right)^{-1}.
    \end{align*}
    It holds that $\Psi(\hat\theta_T) - \Psi(\theta_0) \gtrsim \alpha_*$ with failure probability at most \begin{align*}
         \frac{d}{B \alpha_*} \cdot O\left( \frac{(\beta^2\lambdamax + 1)}{\beta^2\lambdamin} \left( \frac{1}{\|\theta_0\|\|\nabla\loss(\theta_0)\|} + \frac{1}{\|\nabla \loss(\theta_0)\|^2}\right) \right).
    \end{align*}
\end{theorem}

\begin{proof}
    We claim it's enough to bound the probability that $\|\hat{\theta}_T - \bar{\theta}_T\|$ exceeds $\delta$. By the first assumption on $\tau$ in \Cref{eq:tau-constraints}, we may invoke \Cref{lem:euler-approximation}; taking step size \begin{align*}
        \eta \leq \delta \left(\frac{\gamma}{2}\|\theta_0\| + \frac{\gamma^2}{2}\frac{\|\theta_0\|^2}{\|\nabla \loss(\theta_0)\|}\right)^{-1},
    \end{align*}
    we have $\|\Theta_{\eta T} - \bar{\theta}_T\| \leq \delta$. Also by the first constraint on $\tau$, recall that $\|\Theta_{\eta T}\| \geq \gamma\|\theta_0\|$. If it holds that $\|\hat{\theta}_T - \bar{\theta}_T\| \leq \delta$, then \begin{align*}
        \|\hat{\theta}_T\| \geq \|\Theta_{\eta T}\| - \|\hat{\theta}_T - \bar{\theta}_T\| - \|\bar{\theta}_T-\Theta_{\eta T}\| \geq \|\Theta_{\eta T}\| -2\delta.
    \end{align*}
    Moreover, by \cref{lem:gradient-flow}, it must hold that $\alpha_* \leq 2$, and hence \begin{align*}
        \delta \leq \frac{\gamma\|\theta_0\|}{4}.
    \end{align*}
    This implies that $\|\hat{\theta}_T\|\geq \gamma\|\theta_0\|/2$. Since $\Psi(\theta)$ is $\lambdamax^{1/2}/(r\lambdamin^{1/2})$-Lipschitz for $\|\theta\| \geq r$ (see \cref{lem:cosine-lipschitz}), this would imply \begin{align*}
        | \Psi(\hat\theta_T) - \Psi(\Theta_{\eta T})| \leq \frac{4  \lambdamax^{1/2} \delta}{\lambdamin^{1/2} \gamma \|\theta_0\|} = \frac{\alpha_*}{2}.
    \end{align*}
    Hence by \cref{lem:gradient-flow}, $\Psi(\hat\theta_T) - \Psi(\theta_0) \gtrsim \alpha_*$ with at most the same failure probability.

    We decompose the error \begin{align*}
        \bar{\theta}_{t+1} - \hat{\theta}_{t+1} &= \Big( \bar{\theta}_t - \hat{\theta}_t \Big) + \eta \Big( g_t - \nabla \loss(\hat{\theta}_t)\Big) + \eta\Big( \nabla \loss(\hat{\theta}_t) - \nabla \loss(\bar{\theta}_t)\Big)\\
        &= \eta \sum_{j=0}^t \Big(g_j - \nabla \loss(\hat{\theta}_j) \Big) + \eta \sum_{j=0}^t \Big( \nabla \loss(\hat{\theta}_t) - \nabla \loss(\bar{\theta}_t)\Big).
    \end{align*}
    Letting $\xi_t := g_t - \nabla \loss(\hat{\theta}_t)$ for $t \in \{0, \ldots, T-1\}$, and $\CE_t = \sum_{j = 0}^t \xi_j$, and $\CE = \max_{0 \leq t \leq T-1} \|\CE_t\|$, we observe that \begin{align*}
        \|\bar{\theta}_{t+1} - \hat{\theta}_{t+1}\| &\leq \eta \CE + \frac{\lambdamax\beta^2 \eta}{4} \sum_{j=0}^t \|\bar{\theta}_j - \hat{\theta}_j\|.
    \end{align*}
    By discrete Gr\"onwall (\Cref{lem:discrete-gronwall}), it then holds that \begin{align*}
        \| \bar{\theta}_T - \hat{\theta}_T\| \leq \eta \CE \cdot e^{\tau \lambdamax\beta^2/ 4}.
    \end{align*}

    Note that $\{\CE_t\}_{t \geq 0}$ is a martingale w.r.t. $\{\hat{\theta}_t, \{\bm{x}^{(t)}_i\}_i\}_t$, with difference sequence $\{\xi_t\}_{t \geq 1}$. 
    Letting $\E_{t-1}$ denote the expectation conditional to $\sigma(\hat{\theta}_\ell, \bm{x}_i^{(\ell)} : \ell \leq t-1)$, we have the bound
    \begin{align*}
        \E_{t-1}[ \xi_{t} \xi_{t}^\top ] = \E_{t-1}[g_t g_t^\top] - \nabla \loss(\hat{\theta}_t) \nabla\loss(\hat{\theta}_t)^\top \preceq \frac{\beta^2}{B^2} \sum_i \E_{t-1}\Big[ \big(\bm{y} - \sigma(\hat{\theta}_t^\top \bm{x})\big)^2 \cdot \bm{x}\bm{x}^\top\Big] \preceq \frac{\beta^2}{B} \Sigma.
    \end{align*}
    By orthogonality of increments and total expectation, \begin{align*}
        \E \|\CE_t \|^2 = \sum_{j=0}^{t} \E \|\xi_j\|^2 = \E \sum_{j=0}^t \E_{j-1} \|\xi_j\|^2 \leq (t+1) \frac{\beta^2 d \lambdamax}{B}.
    \end{align*}
    Since $\|\CE_t\|$ is a submartingale, the $L^2$ maximal inequality gives
    \begin{align*}
        \Pr\left[\CE \geq \frac{\delta}{\eta e^{\tau \beta^2 \lambdamax/4}}\right] &\leq \frac{\eta^2 e^{\tau \beta^2\lambdamax / 2}}{\delta^2} \E \CE^2\\
        &\leq \frac{4\eta^2 e^{\tau \beta^2\lambdamax / 2}}{\delta^2} \E\|\CE_{T-1}\|^2\\
        &\leq \frac{4\eta \tau \beta^2 \lambdamax e^{\tau \beta^2\lambdamax/2}}{\delta^2} \cdot \frac{d}{B}.
    \end{align*}
   And $\|\hat{\theta}_T - \bar{\theta}_T\| \leq \delta$ with at most this failure probability. 

    We conclude by placing an upper bound on the first term on the product. First note that \begin{align*}
        e^{\tau \beta^2\lambdamax/2} \leq \left(1 + \frac{\gamma \|\theta_0\|}{\|\nabla \loss(\theta_0)\|}\right)^2.
    \end{align*}
    Hence, \begin{align*}
        \frac{\eta \tau \beta^2 \lambdamax e^{\tau \beta^2 \lambdamax/2}}{\delta^2} &\lesssim \frac{\log\left(1 + \frac{\gamma\|\theta_0\|}{\|\nabla\loss(\theta_0)\|} \right) \left(1 + \frac{\gamma\|\theta_0\|}{\|\nabla\loss(\theta_0)\|} \right)^2}{\delta \left(\frac{\gamma}{2}\|\theta_0\| + \frac{\gamma^2}{2}\frac{\|\theta_0\|^2}{\|\nabla \loss(\theta_0)\|}\right)}\\
        &\lesssim \frac{\log\left(1 + \frac{\gamma\|\theta_0\|}{\|\nabla\loss(\theta_0)\|} \right) \left(1 + \frac{\gamma\|\theta_0\|}{\|\nabla\loss(\theta_0)\|} \right)}{\delta \gamma\|\theta_0\|}\\
        &\lesssim \frac{\lambdamax^{1/2}}{\lambdamin^{1/2}\alpha_*} \frac{\log\left(1 + \frac{\gamma\|\theta_0\|}{\|\nabla\loss(\theta_0)\|} \right) \left(1 + \frac{\gamma\|\theta_0\|}{\|\nabla\loss(\theta_0)\|} \right)}{\gamma^2\|\theta_0\|^2}\\
        &\leq \frac{\lambdamax^{1/2}}{\lambdamin^{1/2}\alpha_*} \frac{\left(1 + \frac{\|\theta_0\|}{\|\nabla\loss(\theta_0)\|} \right)}{\gamma \|\theta_0\| \|\nabla\loss(\theta_0)\|}\\
        &\lesssim \frac{(\beta^2\lambdamax + 1)}{\beta^2\lambdamin^{1/2}\lambdamax^{1/2}\alpha_*} \left( \frac{1}{\|\theta_0\|\|\nabla\loss(\theta_0)\|} + \frac{1}{\|\nabla \loss(\theta_0)\|^2}\right).
    \end{align*}
    Combining with the previous bound concludes the proof.
\end{proof}

\subsection{Concise Bounds}

From \cref{thm:stochastic-analysis}, we briefly derive the simplified statement appearing in \cref{sec:overview}. For simplicity, we restrict to the case $\beta = 1$.

\begin{theorem}\label{thm:refined-bounds}
    Let $\mu$ be $\eps$-elliptical along $\on{span}\{\psi, \theta_0\}$, and fix $\beta = 1$. There are constants $C_0, C_1, C_2, C_3$, depending only on the top and bottom eigenvalues of $\Sigma$, such that the following holds. If $\rho > 0$, then for some $\tau > 0$ and any $\eta = \tau/T$, where \begin{align*}
        \tau \leq C_0 \left( \frac{\rho}{\|\theta_0\|}\right),\quad \quad 
        \eta \leq C_1 \left(\frac{\rho}{\|\theta_0\| + 1 + \eps}\right)^3,
    \end{align*}
    it holds that \begin{align*}
        \Psi(\hat{\theta}_T) - \Psi(\theta_0) &\geq \Delta := C_2 \left(\frac{\rho}{\|\theta_0\| + 1 + \eps}\right)^2,
    \end{align*}
    with failure probability at most \begin{align*}
        C_3 \frac{d}{B\Delta}\cdot\left( \frac{1}{\|\theta_0\|\rho} + \frac{1}{\rho^2}\right). 
    \end{align*}
\end{theorem}

\begin{proof}
    For $\eta$ sufficiently small, \cref{thm:stochastic-analysis} gives that $\Psi(\hat{\theta}_t) - \Psi(\hat{\theta}_0) \gtrsim \alpha_*$. We first bound the two quantities defining $\alpha_*$ from below. \begin{align*}
        \frac{\rho}{\lambdamax^{3/2} \beta \|\theta_0\|} \log\left(\frac{\gamma \|\theta_0\|}{\|\nabla \loss(\theta_0)\|} + 1\right) &\overset{(i)}{\geq} \frac{\rho}{\lambdamax^{3/2}} \frac{\left(\frac{\gamma}{\|\nabla \loss(\theta_0)\|}\right)}{\left(\frac{\gamma \|\theta_0\|}{\|\nabla\loss(\theta_0)\|} + 1\right)}\\
        &\geq \frac{\rho\gamma}{\lambdamax^{3/2}(\|\theta_0\| + \|\nabla\loss(\theta_0)\|)}\\
        &\overset{(ii)}{\geq} \frac{\rho}{\lambdamax^{1/2} (\lambdamax  +4) (\|\theta_0\| + \lambdamax^{1/2}\left(\frac{3}{2} + 2\eps\right)) }\\
        &\geq C_2 \frac{\rho}{\|\theta_0\| + 1 + \eps}.
    \end{align*}
    In \emph{(i)}, we used that $\log(x+1) \geq x/(x+1)$ for $x \geq 0$. In \emph{(ii)}, we used \cref{lem:first-order-upper-bound}. Next, after possibly increasing $C_2$ between lines, observe that \begin{align*}
        \frac{\rho^2 \gamma^2}{\lambdamax \left(\frac{2}{\sqrt{3}} \frac{\lambdamax}{\lambdamin} + \frac{\beta^2}{4} \frac{\lambdamax^{3/2}}{\lambdamin^{1/2}} \right)\left(\|\nabla\loss(\theta_0)\| + \gamma \|\theta_0\|\right)^2} &\geq C_2 \frac{\rho^2}{\left(\|\nabla\loss(\theta_0)\| + \|\theta_0\|\right)^2}\\
        &\geq C_2 \left( \frac{\rho}{\|\theta_0\| + 1 + \eps}\right)^2.
    \end{align*}
    We again used \cref{lem:first-order-upper-bound}.
    Note that $\rho \lesssim \lambdamax$ by \cref{lem:rho-upper-bound}, which concludes the lower bound on $\alpha_*$ after possibly updating $C_2$ again. 

    Next, note that \begin{align*}
        \delta \cdot \left(\frac{\gamma}{2}\|\theta_0\| + \frac{\gamma^2}{2} \frac{\|\theta_0\|^2}{\|\nabla\loss(\theta_0)\|}\right)^{-1} &= \frac{2\delta}{\gamma\|\theta_0\|\left(1 + \frac{\gamma\|\theta_0\|}{\|\nabla\loss(\theta_0)\|}\right)}\\
        &= \frac{\lambdamin^{1/2}\alpha_*}{4\lambdamax^{1/2}\left(1 + \frac{\gamma\|\theta_0\|}{\|\nabla\loss(\theta_0)\|}\right)}\\
        &\geq C_1 \frac{\rho^2}{\left(1 + \|\theta_0\| + \eps\right)^2 \left(1 + \frac{\|\theta_0\|}{\|\nabla\loss(\theta_0)\|}\right)}\\
        &\geq C_1 \frac{\rho^2 \|\nabla\loss(\theta_0)\|}{(1 + \|\theta_0\| + \eps)^3}\\
        &\geq C_1 \frac{\rho^3}{(1 + \|\theta_0\| + \eps)^3}.
    \end{align*}
    The last line uses \cref{lem:grad-norm-lower-bound}.
    The bound on $\tau$ is immediate from \cref{eq:tau-constraints}. To bound the failure probability, use the lower bound on $\alpha_*$ and \cref{lem:grad-norm-lower-bound}.
\end{proof}

\section{Approximate Ellipticity}\label{sec:approximate-ellipticity-assumption}

Here, we briefly discuss the approximate ellipticity assumption, and argue that it is somewhat mild. As suggested in \Cref{sec:overview}, spherical and elliptical symmetry are known to have a handful of equivalent definitions. To set the stage, we state one such classic result.

\begin{lemma}[\cite{eaton1986characterization}]
    Suppose the random vector $\bm{z}$ in $\R^d$ has a mean. Assume that for each $v \neq 0$ and each $u$ such that $u^\top v =0$, it holds that \begin{align*}
        \E[\, u^\top \bm{z} \, | \, v^{\top} \bm{z} \,] = 0 \quad \textnormal{a.s.}
    \end{align*}
    Then $\bm{z}$ is spherically symmetric and conversely.
\end{lemma}

\noindent An appropriate generalization for elliptical symmetry follows quite naturally. In light of this characterization, the approximate ellipticity assumption may appear quite strong, even if we only enforce it on a low-dimensional subspace. In a growing line of work, however, \cite{hall1993almost,leeb2013conditional,steinberger2018conditional} show that some notion of ellipticity is actually quite natural, at least on most constant-dimensional subspaces. 

We informally describe the relevant consequence of these works. Let $\bm{z}$ be a random vector with Lebesgue density satisfying $\E \bm{z} = 0$, $\on{Cov} \bm{z} = I$. Let $\mathcal{V}_{d,k}$ denote the Stiefel manifold of orthonormal $k$ frames in $\R^d$, with normalized Haar measure $\nu_{d,k}$. The work of \cite{steinberger2018conditional} exhibits generic monomial moment conditions under which for constant $k$ (or even growing by order $\log d$), there exists a Borel set $G \subset \mathcal{V}_{d,k}$ such that \begin{align}\label{eq:steinberger-leeb}
    \sup_{M \in G} \Pr \left[ \left\| \E\left[ \bm{z} \,\Big|\, M^\top \bm{z} \right] - MM^\top \bm{z} \right\| > t \right] = o(1), \quad \text{and} \quad \nu_{d,k}(G) = 1-o(1).
\end{align}
As an example, if $\bm{z}$ has independent components with bounded densities and bounded 12th marginal moments, then any unitary transformation of $\bm{z}$ satisfies \Cref{eq:steinberger-leeb}. 

For our purposes, it suffices to take only $k = 1$. While, in contrast, our notion of approximate ellipticity imposes an almost-sure condition on conditional expectations, it only does so for the projections of the conditioned random variable onto a two-dimensional subspace (\emph{c.f.} \cref{def:ellipticity}). To see the connection, suppose $\bm{y} = \Sigma^{-1/2} \bm{z}$ satisfies \cref{eq:steinberger-leeb} with $t = \eps$, $G = \mathcal{V}_{d,1}$, and the $o(1)$ replaced by zero. For $v \neq 0$, \begin{align*}
    \E\left[ \bm{z} \,\big| v^\top \bm{z}\right] &= \Sigma^{1/2} \E\left[ \bm{y} \,\Big|\, (\Sigma^{1/2} v)^\top \bm{y}\right] = \Sigma^{1/2} \E\left[ \bm{y} \,\Big|\, \frac{(\Sigma^{1/2} v)^\top \bm{y}}{\sqrt{v^\top \Sigma v}} \right].
\end{align*}
For $u \in \R^d$ such that $u^\top \Sigma v = 0$, we have \begin{align*}
    \left|\E\left[ u^\top \bm{z} \,\big|\, v^\top \bm{y} \right]\right| &= \left| (\Sigma^{1/2}u)^\top \E \left[ \bm{y} \,\Big|\, \frac{(\Sigma^{1/2}v)^\top \bm{y}}{\sqrt{v^\top \Sigma v}}\right] \right|\\
    &= \left| (\Sigma^{1/2}u)^\top \E\left[ \bm{y} \,\Big|\, \frac{(\Sigma^{1/2} v)^\top \bm{y}}{\sqrt{v^\top \Sigma v}} \right]  - \frac{(\Sigma^{1/2}u)^\top (\Sigma^{1/2}v) (\Sigma^{1/2}v)^\top \bm{y}}{v^\top \Sigma v} \right|\\
    &\leq \|\Sigma^{1/2}u\| \cdot \left\| \E \left[ \bm{y} \,\Big|\, \frac{(\Sigma^{1/2}v)^\top \bm{y}}{\sqrt{v^\top\Sigma v}} \right] - \frac{(\Sigma^{1/2}v)(\Sigma^{1/2}v)^\top \bm{y}}{v^\top \Sigma v}\right\| \\
    &\leq \|u\|_{\Sigma} \cdot \eps.
\end{align*}
This recovers \cref{def:ellipticity}. We leave the extension of our analysis to weaker assumptions as an avenue for future work. In the remainder of this section, we prove the claims relevant to the present notion. First, we state an easy observation that simplifies the necessary proofs.

\begin{lemma}\label{lem:ellipticity-transformation}
    Suppose $\bm{x}$ be a random vector in $\R^d$ that is $\eps$-elliptical along a nonzero subspace $E$, and let $T \in \R^{d \times d}$ be full-rank. Then $T \bm{x}$ is $\eps$-elliptical along $T^{-\top} E$.
\end{lemma}

\begin{proof}
    Suppose $\bm{x}$ has covariance $\Sigma$, which is full-rank by the definition. Let $z \in \R^d$ and $w \in T^{-\top} E \setminus \{0\}$ such that $z^\top T\Sigma T^{\top}w = 0$. Since $T^\top w \in E \setminus \{0\}$, we have \begin{align*}
        \E\left[\, z^\top T\bm{x} \,\big|\, w^\top T\bm{x} \,\right] &= \E\left[\, (T^\top z)^\top \bm{x} \,\big|\, (T^\top w)^\top \bm{x} \,\right]\\
        &\leq \|T^\top z\|_{\Sigma} \cdot \eps\\
        &= \|z\|_{T\Sigma T^\top} \cdot \eps.
    \end{align*}
    Conclude by noting $T\bm{x}$ has covariance $T\Sigma T^\top$.
\end{proof}

\subsection{Proof of \Cref{lem:arccos-loss}}

\restatelemma{lem:arccos-loss}
\begin{proof}
    First suppose $\Sigma = I$.
    By \cite[Theorem~1]{eaton1986characterization}, the law of $\bm{y}$ is invariant under unitary transformation. By Lebesgue-Radon-Nikodym, so is the density. For convenience, we denote it as a map $\R \to \R$. By a change of variables, \begin{align*}
        \Pr[\,\on{sgn} \phi^\top \bm{y} \neq \on{sgn} \psi^\top \bm{y}\,] &= \int_0^\infty r^{k-1} \int_{S^{k-1}} \frac{d\mu}{dm}(r) \cdot \chi(\on{sgn} \phi^\top \omega \neq \on{sgn} \psi^\top \omega)\,  d\sigma(\omega)dr\\
        &= \frac{1}{\pi} \arccos \circ \cos(\phi, \psi) \cdot \int_0^\infty \frac{d\mu}{dm}(r) \cdot r^{k-1} \sigma(S^{k-1}) \, dr,
    \end{align*}
    and the second factor integrates to one. Since $\Sigma^{-1/2}\bm{y}$ is $0$-spherical with Lebesgue density, the definition of $\cos_\Sigma$ gives the general claim.
\end{proof}

\subsection{Proof of \Cref{prop:arccos-loss-approx}}

 \restateproposition{prop:arccos-loss-approx}

\begin{proof}
    Again note that it suffices to consider the case where $\bm{x}$ is $\eps$-spherical. Indeed, $\Sigma^{-1/2} \bm{x}$ has a density $d\mu_I/dm : \R^d \to \R$ satisfying \begin{align*}
        \Big|\widehat{\frac{d\mu_I}{dm}}(\xi)\Big| = \Big|\widehat{\frac{d\mu}{dm}}(\Sigma^{-1/2} \xi)\Big| \leq \frac{C}{(1+ \lambdamax^{-1/2}\|\xi\|)^k} \leq \frac{C \cdot \max\{1, \lambdamax^{k/2}\}}{(1+\|\xi\|)^k}.
    \end{align*}
    Moreover, $\Sigma^{-1/2}\bm{x}$ is $\eps$-spherical on $\on{span}\{\Sigma^{1/2}\phi, \Sigma^{1/2}\psi\}$. Assuming the desired result in the case of identity covariance, we may conclude by noting $\cos_I(\Sigma^{1/2}\phi, \Sigma^{1/2}\psi) = \cos_\Sigma(\phi, \psi)$.

    We now proceed under the assumption of $\eps$-spherical symmetry. Let $E := \on{span}\{\phi, \psi\}$. First note that the event under consideration depends only on $\bm{z} = \on{Proj}_E \bm{x}$.
    Also note that we can take $\phi, \psi, \bm{z}$ supported on $\R^2$ without loss of generality. Indeed, choose a unitary $U$ so that $U\on{span}\{\phi, \psi\} = \on{span}\{e_1, e_2\}$. For orthogonal $v, w \in \on{span}\{e_1, e_2\}$, $U^\top v, U^\top w$ are orthogonal in $\on{span}\{\phi, \psi\}$, and hence \begin{align*}
        \left| \E\left[ v^\top U\bm{z} \, \big| \, w^\top U\bm{z} \right]\right| &= \left| \E\left[ (U^\top v)^\top \bm{z} \, \big| \, (U^\top w)^\top \bm{z} \right] \right|\\
        &= \left| \E\left[ (U^\top v)^\top \bm{x} \, \big| \, (U^\top w)^\top \bm{x} \right] \right| \leq \|v\|\eps. 
    \end{align*}
    Note that $\bm{z}$ now has the density $f : \R^2 \to \R$ given by \begin{align*}
        f(z) = \int_{\R^{d-2}} \frac{d\mu_I}{dm}(z, y) \,dy.
    \end{align*}
    The Fourier transform becomes $\hat{f}(\xi) = \widehat{(d\mu/dm)}(\xi, 0)$, so the decay assumptions are inherited.

    We first observe an approximate version of Eaton's ellipticity characterization. For $\xi \in \R^2 \setminus \{0\}$ and $w \in \R^2$ such that $w^\top \xi = 0$, we have
    \begin{align*}
        w^\top \nabla \hat{f}(\xi) &= -i2\pi \E w^\top \bm{z} \exp\left( -i2\pi \langle \xi, \bm{z} \rangle \right)\\
        &= -i2\pi \E \left[ \exp\left(-i2\pi \langle \xi, \bm{z}\rangle \right) \cdot \E\left[ w^\top \bm{z} \, \big| \, \langle \xi, \bm{z}\rangle \right] \right].
    \end{align*}
    By Jensen's, \begin{align*}
        | w^\top \nabla \hat{f}(\xi)| \leq 2\pi\eps\|w\|.
    \end{align*}
    Now take any $\xi'$ in $\R^2$, so that $\|\xi'\| = \|\xi\| =: r$, which is positive by assumption. There exists a $C^\infty$ map $\zeta : (0, 1) \to r \cdot S^1$ so that for some $0 < t_1 < t_2 < 1$, \begin{align*}
        \zeta(t_1) = \xi, \quad \zeta(t_2) = \xi'.
    \end{align*}
    Moreover, one may choose $\zeta, t_1, t_2$ so that $\sup_t \|\partial_t \zeta\| \leq C'r$ where $\partial_t \zeta(t)$ denotes the vector of derivatives wrt $t$, and $C' \geq 1$ is an absolute constant. Since $\|\zeta\|^2$ is identically $r^2$, we have \begin{align*}
        0 = \partial_t \|\zeta(t)\|^2 = 2 \langle \zeta(t), \partial_t \zeta(t) \rangle.
    \end{align*}
    By chain rule, however, $\partial_t(\hat{f} \circ \zeta) = \nabla \hat{f}(\zeta(t))^\top (\partial_t \zeta(t))$,
    so $|\partial_t (\hat{f} \circ \zeta)| \leq 2\pi\eps C'r$ for all $t$.
    By MVT, there is some $t^* \in (t_1, t_2)$ so that \begin{align*}
        \hat{f}(\xi') - \hat{f}(\xi) = \hat{f}(\zeta(t_2)) - \hat{f}(\zeta(t_1)) = (t_2 - t_1)\partial_{t}(\hat{f} \circ \zeta)(t^*),
    \end{align*}
    which gives $|\hat{f}(\xi') - \hat{f}(\xi)| \leq 2\pi C' r\eps$.
    
    For an integrable $h :\R^2 \to \R$, we define the spherical average $h_\circ : \R^2 \to \R$ as the $L^1$ function \begin{align*}
        h_\circ(x) = \frac{1}{\sigma(S^1)} \int_{S^1} h(\|x\| \cdot\omega) d\sigma(\omega).
    \end{align*}
    Here, we let $d\sigma$ denote the surface measure on $S^1$. 
    It is easy to check that if $h\geq 0$ and $|h|_1 = 1$, then $h_\circ \geq 0$ and $|h_\circ|_1=1$. Moreover, $\CF[f_\circ] = \CF[f]_\circ$, so we can write $\hat f_\circ$ without ambiguity. Indeed, by a change of measure and Fubini, \begin{align*}
        \CF[f]_\circ(\xi) &= \frac{1}{\sigma(S^1)} \int_{S^1} \hat{f}(\|\xi\|\cdot \omega)d\sigma(\omega)\\
        &= \frac{1}{\sigma(S^1)} \int_{S^1} \int_{\R^2} e^{-i2\pi \langle \|\xi\| \omega, x\rangle}{f}(x) dx d\sigma(\omega)\\
        &= \frac{1}{\sigma(S^1)} \int_{S^1} \int_{0}^\infty \int_{S^{1}} re^{-i2\pi \langle \|\xi\| \omega, r \tilde{\omega}\rangle}{f}(r \tilde{\omega}) d\sigma(\tilde{\omega})dr d\sigma(\omega)\\
        &= \frac{1}{\sigma(S^1)} \int_0^\infty \int_{S^{1}} r f(r\tilde{\omega}) \left( \int_{S^1} e^{-i2\pi\langle \|\xi\|\omega, r\tilde{\omega}\rangle} d\sigma(\omega)\right) d\sigma(\tilde{\omega})dr\\
        &\overset{(i)}{=} \frac{1}{\sigma(S^1)} \int_0^\infty \int_{S^{1}} r f(r\tilde{\omega}) \left( \int_{S^1} e^{-i2\pi\langle r\omega, \xi\rangle} d\sigma(\omega)\right) d\sigma(\tilde{\omega})dr\\
        &= \int_0^\infty \left( \frac{1}{\sigma(S^1)} \int_{S^{1}} f(r\tilde{\omega}) d\sigma(\tilde{\omega}) \right)\cdot \left( \int_{S^1} e^{-i2\pi\langle r\omega, \xi\rangle} d\sigma(\omega)\right) \cdot rdr\\
        &\overset{(ii)}{=} \int_0^\infty \int_{S^1} re^{-i2\pi\langle r\omega, \xi\rangle} f_\circ(r \omega) d\sigma(\omega) dr,
    \end{align*}
    from which we conclude by an additional change of measure.
    In \emph{(i)}, we used the fact that $\sigma$ is rotation invariant and chose a unitary $U$ such that $U^\top \tilde{w} = \xi/\|\xi\|$. In \emph{(ii)}, we used the fact that for $r \geq 0$ and $\omega \in S^{1}$, $f_\circ(r\omega)$ depends only on $r$.

    Next, we bound the total variation between $f$, and its spherical average. Note first that $f_\circ$ inherits the Fourier decay of $f$; by Jensen's, \begin{align}\label{eq:spherical-average-decay}
        |\hat f_\circ(\xi)| &\leq \frac{1}{\sigma(S^1)} \int_{S^1} |\hat{f}(\|\xi\| \cdot \tilde{\omega})| d\sigma(\tilde{\omega}) \leq C(1+\|\xi\|)^{-k}.
    \end{align}
    It therefore holds that $\hat f, \hat f_\circ \in L^1$. A standard argument then implies that $f, f_\circ$ are equal almost everywhere to the inversions of $\hat f$ and $\hat f_\circ$, respectively.\footnote{Let $g \in L^1$ such that $g= \CF^{-1}(\hat{f})$. Then for any Schwartz function $\varphi$ on $\R^2$, Fubini gives $\int f \hat{\varphi} = \int \hat{f} \varphi = \int g \hat{\varphi}$. Since the Fourier transform is a bijection on $\mathcal{S}$, approximation of the identity gives equality almost everywhere.}
    Let $B = B_0(r_1)$, and $C = B_0(r_2)$ for $r_1, r_2 > 0$ to be chosen later. 
    \begin{align*}
        \|f-f_\circ\|_1 &\leq \int_B \left| \int_{\R^2} e^{i2\pi \langle \xi, x\rangle} \Big( \hat{f}(\xi)-\hat{f}_\circ(\xi)\Big) d\xi \right| dx + \int_{B^c} f(x) + f_\circ(x) dx\\
        &\leq \int_B \left( \int_{C} |\hat{f}(\xi)-\hat{f}_\circ(\xi)| d\xi + \int_{C^c} |\hat f(\xi)| + |\hat f_\circ (\xi)| d\xi\right) dx + \int_{B^c} f(x) + f_\circ(x) dx.
    \end{align*}
    We start with the first integrand. For $r > 0$ and $\omega \in S^1$, we have \begin{align*}
        |\hat{f}(r\omega) - \hat f_\circ(r\omega)| \leq \frac{1}{\sigma(S^1)} \int |\hat{f}(r\omega) - \hat{f}(r\tilde{\omega})| d\sigma(\tilde{\omega}) \leq 2\pi C' r\eps.
    \end{align*}
    Next, note that, \begin{align*}
        \int_{B^c} f_\circ(x)dx &= \int_{r_1}^\infty \int_{S^1} \frac{r}{\sigma(S^1)} \int_{S^1} f(r \tilde{\omega}) d\sigma(\tilde{\omega}) d\sigma(\omega) dr\\
        &= \int_{r^1}^\infty \int_{S^1} r f(r\tilde{\omega}) d\sigma(\tilde{\omega})dr = \int_{B^c} f(x)dx.
    \end{align*}
    Along with another application of Jensen's (akin to \Cref{eq:spherical-average-decay}), we obtain \begin{align*}
        \|f-f_\circ\|_1 &\leq \pi r_1^2 \left( 2\pi^2 r_2^3 C' \eps + 2 \int_{C^c} |\hat f(\xi)| d\xi \right) + 2\int_{B^c} f(x) dx.
    \end{align*}

    We handle these remaining integrals individually. Noting that $\E \|\bm{z}\|^2 = \on{tr}(U \Proj_E U^\top) = 2$, Chebyshev gives $2\int_{B^c} f(x) dx \leq {4}{r_1^{-2}}$.
    Lastly, by our assumptions on Fourier decay, \begin{align*}
        \int_{C^c} |\hat{f}(\xi)|d\xi \leq C\int_{\|\xi\| \geq r_2} \frac{1}{(1+\|\xi\|)^k}d\xi \leq \frac{C\sigma(S^1)}{(k-2)r_2^{k-2}}.
    \end{align*}
    Updating our bound, we have \begin{align*}
        \|f-f_\circ\|_1 &\leq 2\pi^3 r_1^2 r_2^3 C' \eps + \frac{4C \pi^2}{k-2} r_1^2 r_2^{2-k} + 4r_1^{-2}.
    \end{align*}
    To finish, we optimize $r_1, r_2$. By \cref{lem:non-convex-small-value}, for $\alpha = 3/(k+1)$, \begin{align*}
        \|f-f_0\|_1 \lesssim A \cdot (2 \pi^{3} C' \eps)^{\frac{1}{2}- \frac{\alpha}{2}},
    \end{align*}
    where \begin{align*}
        A &= 2 \left(\frac{4C \pi^2}{3}\right)^{\frac{\alpha}{2}} + \left(\frac{4C\pi^2}{k-2}\right)^{\frac{1}{2}} \left(\frac{4C\pi^2}{3}\right)^{\frac{\alpha}{2}-\frac{1}{2}}\\
        &\overset{(i)}{\leq} 2 \left(\frac{4\pi^2}{3}\right) C^\alpha + (4\pi^2)^{\frac{1}{2}}\left(\frac{4\pi^2}{3}\right)^{\frac{\alpha}{2}-\frac{1}{2}} C^{\frac{\alpha}{2}} \lesssim C^\alpha.
    \end{align*}
    In \emph{(i)}, we used that $C \geq 1$ and $\alpha \in (0, 1)$. We may conclude that \begin{align*}
        \|f - f_\circ\|_1 \lesssim C^{\frac{3}{k+1}} \cdot \eps^{\frac{1}{2}\left(1- \frac{3}{k+1}\right)}.
    \end{align*}
    Since a random vector with density $f_\circ$ satisfies \Cref{lem:arccos-loss}, this completes the proof.
\end{proof}

\section{Additional Computations}

\subsection{Cosine Similarity Derivatives}

\begin{lemma}\label{lem:cosine-grad-hess}
    For $u, v \neq 0$, and positive definite $\Sigma$, \begin{align*}
        \nabla_u \cos_\Sigma(u, v) &= \frac{1}{\|u\|_\Sigma \|v\|_\Sigma} \cdot \Sigma \on{Proj}_{u^\perp}^\Sigma (v),\\
        \nabla_u^2 \cos_\Sigma(u, v) &= 3 \left(\frac{u^\top \Sigma v}{\|u\|_\Sigma^5 \|v\|_\Sigma}\right) \Sigma uu^\top \Sigma - \left(\frac{u^\top \Sigma v}{\|u\|_\Sigma^3 \|v\|_\Sigma}\right)\Sigma - \left(\frac{1}{\|u\|_\Sigma^3 \|v\|_\Sigma}\right) \left(\Sigma uv^\top \Sigma + \Sigma vu^\top \Sigma\right).
    \end{align*}
\end{lemma}

\begin{proof}
    To start, assume both computations hold when $\Sigma = I$. By chain rule on the Jacobian, \begin{align*}
        J\Big(\cos_\Sigma(\cdot, v)\Big)(u) = J\Big(\cos_I(\Sigma^{1/2}(\cdot), \Sigma^{1/2}v)\Big)(u) = J\Big(\cos_I(\cdot, \Sigma^{1/2}v)\Big)(\Sigma^{1/2}u) \circ \Sigma^{1/2}.
    \end{align*}
    Taking the transpose, \begin{align*}
        \nabla_u \cos_\Sigma(u, v) &= \frac{1}{\|\Sigma^{1/2}u\| \|\Sigma^{1/2} v\|} \cdot \Sigma^{1/2} \on{Proj}_{(\Sigma^{1/2}u)^\perp}^I(\Sigma^{1/2} v)\\
        &= \frac{1}{\|\Sigma^{1/2}u\| \|\Sigma^{1/2} v\|} \cdot \Sigma^{1/2} \left( \Sigma^{1/2}v -  \frac{\langle \Sigma^{1/2}v, \Sigma^{1/2}u\rangle}{\|\Sigma^{1/2} u\|^2} \Sigma^{1/2} u\right)\\
        &= \frac{1}{\|u\|_\Sigma \|v\|_\Sigma} \cdot \Sigma \left(v - \frac{\langle v, u\rangle_\Sigma}{\|u\|_\Sigma^2}u\right).
    \end{align*}
    With another application of chain rule, \begin{align*}
        J\Big( \nabla \cos_\Sigma(\cdot, v) \Big)(u) &= J\Big( \Sigma^{1/2} \circ \Big[J\big(\cos_I(\cdot, \Sigma^{1/2}v)\big)(\Sigma^{1/2} \cdot)\Big]^\top\Big)(u)\\
        &= \Sigma^{1/2} \circ J\Big( \Big[J\big(\cos_I(\cdot, \Sigma^{1/2}v)\big)(\Sigma^{1/2} \cdot)\Big]^\top\Big)(u)\\
        &= \Sigma^{1/2} \circ J\Big(J\big(\cos_I(\cdot, \Sigma^{1/2}v)\big)^\top \Big)(\Sigma^{1/2}u) \circ \Sigma^{1/2}.
    \end{align*}
    Hence, \begin{align}\label{eq:hessian-isotropic-reduction}
        \nabla_u^2 \cos_\Sigma(u,v) = \Sigma^{1/2} \circ [\nabla^2 \cos_I(\cdot, \Sigma^{1/2}v)](\Sigma^{1/2}u) \circ \Sigma^{1/2}.
    \end{align}
    From this, one may verify the final claim. 

    We now confirm the computations when $\Sigma = I$. The gradient computation is standard---for instance, using total differentiability and the directional calculation of \cite[Proposition~F.3]{geng2025delta}. We compute the Hessian in a similar fashion. Let $\tilde v = v/\|v\|$, $z \in \R^d$, $h \in \R$, and $\Psi(\cdot) = \cos_I(\cdot, v)$. \begin{align*}
        \nabla \Psi(u +hz) &= \frac{\tilde{v}\|u + hz\|^2 - \langle u + hz, \tilde{v}\rangle(u + hz)}{\|u + hz\|^3}\\
        &= \frac{\tilde{v}}{\|u + hz\|} - \frac{u \langle u, \tilde{v}\rangle}{\|u + hz\|^3} - \frac{hu\langle z, \tilde{v}\rangle}{\|u + hz\|^3}-\frac{hz \langle u, \tilde{v}\rangle}{\|u + hz\|^3} - \frac{h^2 z \langle z, \tilde{v}\rangle}{\|u+hz\|^3}.
    \end{align*}
    Differentiating in $h$ and evaluating at zero, one obtains \begin{align*}
        \restr{D_t \,\nabla \Psi(u + hz)}{h=0} = \frac{3 u \langle u, \tilde{v}\rangle \langle u,z\rangle}{\|u\|^5} - \frac{\tilde{v}\langle u, z\rangle}{\|u\|^3}-\frac{u\langle \tilde{v}, z\rangle}{\|u\|^3} - \frac{\langle u, \tilde{v}\rangle z}{\|u\|^3}.
    \end{align*}
    Hence,
    \begin{align*}
        \nabla^2 \Psi(u) = 3 \left(\frac{\langle u, v \rangle}{\|u\|^5 \|v\|}\right) uu^\top - \left(\frac{\langle u, v \rangle}{\|u\|^3 \|v\|}\right)I - \left(\frac{1}{\|u\|^3 \|v\|}\right) \left(uv^\top + vu^\top\right).
    \end{align*}
\end{proof}

\subsection{Inequalities}

We first state the computation needed for \cref{prop:arccos-loss-approx}.

\begin{lemma}\label{lem:non-convex-small-value}
    Let $g :(0, \infty)^2 \to \R$ be given by \[g(x, y) = ax^2 y^3 + b x^2 y^{2-k} + c x^{-2},\]
    for $a,b,c > 0$ and integer $k \geq 3$.
    It holds that for $\alpha = 3/(k+1)$, \begin{align*}
        \inf_{x,y \in (0, \infty)^2} g(x, y) \leq c^{\frac{1}{2}} a^{\frac{1}{2} - \frac{\alpha}{2}} \cdot \left( 2 \left( \frac{b(k-2)}{3}\right)^{\frac{\alpha}{2}} + b^{\frac{1}{2}} \left(\frac{b(k-2)}{3}\right)^{\frac{\alpha}{2}-\frac{1}{2}}\right).
    \end{align*}
\end{lemma}

\begin{proof}
    Consider points \begin{align*}
        y_0 = \left(\frac{b(k-2)}{3a}\right)^{\frac{1}{k+1}}, \quad \quad x_0 = \left( \frac{c}{ay_0^3 + b y_0^{2-k}}\right)^{\frac{1}{4}}.
    \end{align*}
    For $\alpha = 3/(k+1)$, one computes \begin{align*}
        &g(x_0, y_0) = a \left( \frac{c}{a \left( \frac{b(k-2)}{3a} \right)^\alpha + b\left( \frac{b(k-2)}{3a}\right)^{\alpha - 1}} \right)^{\frac{1}{2}} \left(\frac{b(k-2)}{3a}\right)^{\alpha} \\
        &\quad +b \left( \frac{c}{a \left( \frac{b(k-2)}{3a} \right)^\alpha + b\left( \frac{b(k-2)}{3a}\right)^{\alpha - 1}} \right)^{\frac{1}{2}} \left(\frac{b(k-2)}{3a}\right)^{\alpha - 1} + c \left( \frac{c}{a \left( \frac{b(k-2)}{3a} \right)^\alpha + b\left( \frac{b(k-2)}{3a}\right)^{\alpha - 1}} \right)^{-\frac{1}{2}}\\
        &= a^{1-\alpha} c^{\frac{1}{2}} \left( \frac{a^{\alpha - 1}}{ \left( \frac{b(k-2)}{3} \right)^\alpha + b\left( \frac{b(k-2)}{3}\right)^{\alpha - 1}} \right)^{\frac{1}{2}} \left(\frac{b(k-2)}{3}\right)^{\alpha} \\
        &\quad + a^{1-\alpha}bc^{\frac{1}{2}} \left( \frac{a^{\alpha - 1}}{\left( \frac{b(k-2)}{3} \right)^\alpha + b\left( \frac{b(k-2)}{3}\right)^{\alpha - 1}} \right)^{\frac{1}{2}} \left(\frac{b(k-2)}{3}\right)^{\alpha - 1} + c^{\frac{1}{2}} \left( \frac{a^{\alpha - 1}}{ \left( \frac{b(k-2)}{3} \right)^\alpha + b\left( \frac{b(k-2)}{3}\right)^{\alpha - 1}} \right)^{-\frac{1}{2}}.
    \end{align*}
    In particular, \begin{align*}
        g(x_0, y_0) &\leq c^{\frac{1}{2}} a^{\frac{1}{2}- \frac{\alpha}{2}} \cdot A,
    \end{align*}
    where \begin{align*}
        A &= \left( \frac{1}{\left(\frac{b(k-2)}{3}\right)^{\alpha} + b\left(\frac{b(k-2)}{3}\right)^{\alpha - 1}} \right)^{\frac{1}{2}} \left(\frac{b(k-2)}{3}\right)^\alpha \\
        &\quad + b\left( \frac{1}{\left(\frac{b(k-2)}{3}\right)^{\alpha} + b\left(\frac{b(k-2)}{3}\right)^{\alpha - 1}}\right)^{\frac{1}{2}} \left(\frac{b(k-2)}{3}\right)^{\alpha - 1} + \left( \frac{1}{\left(\frac{b(k-2)}{3}\right)^{\alpha} + b\left(\frac{b(k-2)}{3}\right)^{\alpha - 1}}\right)^{-\frac{1}{2}}\\
        &\leq \left( \frac{1}{\left(\frac{b(k-2)}{3}\right)^{\alpha}} \right)^{\frac{1}{2}} \left(\frac{b(k-2)}{3}\right)^\alpha + b\left( \frac{1}{b\left(\frac{b(k-2)}{3}\right)^{\alpha - 1}}\right)^{\frac{1}{2}} \left(\frac{b(k-2)}{3}\right)^{\alpha - 1} + \left( \frac{1}{\left(\frac{b(k-2)}{3}\right)^{\alpha}}\right)^{-\frac{1}{2}}\\
        &= 2\left( \frac{b(k-2)}{3}\right)^{\frac{\alpha}{2}} + b^{\frac{1}{2}}\left(\frac{b(k-2)}{3}\right)^{\frac{\alpha}{2}-\frac{1}{2}}.
    \end{align*}
\end{proof}

\begin{lemma}\label{lem:hessian-top-eigenvalue}
    For $u, v \neq 0$, and positive definite $\Sigma$, \begin{align*}
        \|\nabla_u^2 \cos_\Sigma(u, v)\| \leq \frac{2}{\sqrt{3}} \cdot \frac{\lambdamax(\Sigma)}{\|u\|_\Sigma^2}.
    \end{align*}
\end{lemma}

\begin{proof}
    Via \Cref{eq:hessian-isotropic-reduction}, we may again reduce to the identity case. Indeed, \begin{align*}
        \|\nabla_u^2 \cos_\Sigma(u,v)\| \leq \|\Sigma^{1/2}\|^2 \cdot \|[\nabla^2 \cos_I(\cdot, \Sigma^{1/2}v)](\Sigma^{1/2}u)\| \leq \frac{2}{\sqrt{3}} \cdot \frac{\|\Sigma\|}{\|u\|_{\Sigma}^2}.
    \end{align*}
    Now let $H = \nabla_u^2 \cos_I(u, v)$. If $u,v$ are collinear, then there are real numbers $r_1, r_2$ and a unit vector $w$ such that \begin{align*}
        H = \frac{r_1 r_2}{|r_1|^3 |r_2|} \left(ww^\top - I\right).
    \end{align*}
    So it suffices to assume $u$ and $v$ are independent. For $w \in S^{d-1}$ such that $\langle w, u\rangle = 0$ and $\langle w, v\rangle = 0$, we see \begin{align*}
        Hw = - \left(\frac{u^\top v}{\|u\|^3 \|v\|}\right) w.
    \end{align*}
    Next, note that $\on{span}\{u, v\}$ is $H$-invariant. We write the restriction of $H$ with respect to this subspace. Take $\tilde u = u/\|u\|$, $\tilde v = v /\|v\|$, and consider the orthonormal basis $\{\tilde u, \tilde w\}$ for $w = \on{Proj}_{(\tilde{u})^\perp} \tilde{v}$, and $\tilde{w} = w/{\|w\|}$. We obtain \begin{align*}
        H\tilde{u} = -\frac{\|w\|}{\|u\|^2} \tilde{w}, \quad \quad H\tilde{w} = -\frac{\langle \tilde{u}, \tilde{v}\rangle}{\|u\|^2}\tilde{w} - \frac{\|w\|}{\|u\|^2} \tilde{u},
    \end{align*}
    using $\|w\|^2 = 1-\langle\tilde{v}, \tilde{u}\rangle^2$ in the latter computation.
    The associated block therefore has the form \begin{align*}
        A = \frac{1}{\|u\|^2}\begin{pmatrix}
            0 &-\sqrt{1-\alpha^2}\\
            -\sqrt{1-\alpha^2} &-\alpha.
        \end{pmatrix}, \quad \det(\lambda - A) = \lambda^2 + \lambda \alpha- (1-\alpha^2),
    \end{align*}
    for $0 \leq \alpha \leq 1$. The associated roots are \begin{align*}
        \left\{ \frac{-\alpha + \sqrt{4-3\alpha^2}}{2}, \frac{-\alpha -\sqrt{4-3\alpha^2}}{2}\right\}.
    \end{align*}
    Taking $r(\alpha)$ to be either of the roots above, the pairs $(\alpha, r(\alpha))$ give points on the ellipse defined by $\alpha^2 + r\alpha + r^2 = 1$. Observe that $r$ can satisfy this equality only if the discriminant in $\alpha$ is nonnegative, which is true only if $|r| \leq 2/\sqrt{3}$. We conclude that the spectral radius of $H$ is bounded by $2/(\|u\|^2 \sqrt{3})$.
\end{proof}

\begin{lemma}\label{lem:cosine-lipschitz}
    For positive definite $\Sigma$ and $v \neq 0$, $\cos(\cdot, v)$ is $\lambdamax^{1/2}/r\lambdamin^{1/2}$-Lipschitz on $\R^d \setminus B_r(0)$.
\end{lemma}

\begin{proof}
    Take $\|u\|, \|u'\| \geq r$. One has \begin{align*}
        |\cos_\Sigma(u, v) - \cos_\Sigma(u', v)| &= \left| \left\langle \frac{u}{\|u\|_{\Sigma}} - \frac{u'}{\|u'\|_\Sigma}, \frac{v}{\|v\|_\Sigma}\right\rangle_\Sigma \right|\\
        &\leq \left\| \frac{u}{\|u\|_{\Sigma}} - \frac{u'}{\|u'\|_\Sigma}\right\|_{\Sigma}\\
        &\leq \frac{1}{\|u\|_\Sigma \land \|u'\|_\Sigma} \cdot \|u-u'\|_\Sigma.
    \end{align*}
    To conclude, note that $\|u-u'\|_\Sigma \leq \lambdamax^{1/2}\|u-u'\|$ and $\|u\|_\Sigma \land \|u'\|_\Sigma \geq \lambdamin^{1/2}r$.
\end{proof}

\noindent We also make use of a well-known discrete variant of the Gr\"onwall inequality.

\begin{lemma}\label{lem:discrete-gronwall}
    Let $(a_n)_0^N$, $(b_n)_0^N$ and $c$ be nonnegative reals with $a_0 = 0$, such that \begin{align*}
        a_{n+1} \leq \sum_{k=0}^n b_k a_k + c.
    \end{align*}
    Then for $n \leq N-1$, it holds that \begin{align*}
        a_{n+1} \leq c \cdot \exp \sum_{k=0}^{n} b_k.
    \end{align*}
\end{lemma}

\end{document}